\documentclass{research}

\title{Sharpen Your Flow:\\
Sharpness-Aware Sampling for Flow Matching}

\author[1,2]{Aditi Gupta}
\author[3,4]{Soon Hoe Lim$^{*}$}
\author[5]{Annan Yu}
\author[1,2]{N. Benjamin Erichson$^{*}$}

\affiliation[1]{Lawrence Berkeley National Laboratory}
\affiliation[2]{International Computer Science Institute}
\affiliation[3]{Department of Mathematics, KTH Royal Institute of Technology}
\affiliation[4]{Nordita, KTH Royal Institute of Technology and Stockholm University}
\affiliation[5]{Center for Applied Mathematics, Cornell University}

\contribution[*]{Equal Mentoring}

\abstract{Flow matching models generate samples by numerically integrating a learned velocity field, with each integration step requiring a neural network evaluation. Fast generation therefore requires using a small fixed evaluation budget effectively: the key question is not only how to integrate the flow, but where the sampler should spend its steps. We propose SharpEuler, a training-free sampler that profiles a pretrained model offline by estimating where the learned velocity field changes most rapidly along calibration trajectories. This finite-difference estimate defines a solver-aware sharpness profile, which is smoothed and converted by a quantile transform into a timestep grid for any desired inference budget. At test time, sampling remains ordinary Euler integration with the same number of model evaluations as a uniform schedule. We justify SharpEuler using three principles: a numerical principle identifying trajectory acceleration as the leading source of Euler discretization error, a variational principle deriving sharpness-based power-law timestep densities, and a statistical guarantee showing that the finite-sample calibrated sampler is stable at the terminal distribution level.  Our experiments show that SharpEuler improves sample quality at fixed budgets, reducing inter-mode leakage and increasing mode coverage.}

\code{\url{https://github.com/aditiii12/SharpEuler}}
\correspondence{N. Benjamin Erichson @ \email{erichson@lbl.edu}}

\begin{document}

\maketitle

\section{Introduction}

Flow-based generative models have become a powerful tool for modeling high-dimensional distributions~\cite{lipman2023flow,liu2023flow,albergo2023stochastic}.
Their applications range from text-to-image~\cite{esser2024scaling} and text-to-video generation~\cite{jinpyramidal} to scientific domains such as time series modeling~\cite{hu2024flowts, lim2026flow}, fluid flow forecasting~\cite{limelucidating}, molecule generation~\cite{morehead2026zatom}, materials design~\cite{sriram2024flowllm}, and protein structure generation~\cite{didiscaling}. As an instance of dynamical measure transport, sampling with a flow matching model requires integrating a learned velocity field that transports samples from a simple noise distribution to the data distribution.

This makes sampling a numerical integration problem. In classical numerical analysis, accuracy is often improved using higher-order or adaptive solvers~\cite{hairer1993solving,griffiths2010numerical,butcher2016numerical}. For neural generative models, however, each function evaluation is a forward pass through a large network~\cite{lipman2024flow}, and these evaluations dominate sampling cost in modern image and video models~\cite{esser2024scaling}. Thus the relevant quantity is not only asymptotic solver accuracy, but sample quality per neural network evaluation. For this reason, many practical samplers use simple one-stage methods with a fixed budget. Forward Euler is attractive because a sampler with budget \(B\) uses exactly \(B\) model evaluations. If we do not add solver stages or use an online adaptive method with sample-dependent cost, the remaining question is then:
\begin{tcolorbox}[
  colback=parchment,
  colframe=sandborder,
  boxrule=0.3pt,
  left=8pt,
  right=8pt,
  top=1pt,
  bottom=1pt
]
\centering
\emph{Where should a flow matching sampler spend its finite Euler budget?}
\end{tcolorbox} Uniform and shifted schedules are simple and often effective, but they are not derived from the learned dynamics being integrated. Recent work has shown that timestep allocation can strongly impact sample quality, especially in the few-evaluation regime~\cite{watson2021learning,xue2024accelerating,williams2024score,sabour2024align}. Other schedulers use information-theoretic criteria, such as
entropy-based time reparameterizations that motivates the Entropic Time Scheduler (ETS) of \cite{stancevic2025entropic}. Our perspective is complementary: for a fixed Euler solver, the schedule should be based on where the learned vector field is changing rapidly. See also App. \ref{app:relatedwork} for a detailed discussion of related work.

We propose SharpEuler, a training-free, solver-aware schedule calibration
method for pretrained flow matching models. The method estimates where the
learned trajectory is sharp (velocity field changes rapidly) rather than choosing this region by hand. In an offline calibration stage, we run fine-grained Euler trajectories on a  set of representative prompts or samples, record consecutive velocity evaluations, and estimate how rapidly the vector field changes along the trajectory. Averaging and smoothing the finite-difference estimated accelerations gives a
time-dependent sharpness profile. At inference time, we convert this profile into a monotone grid and run the same Euler sampler
on that grid. The method does not change the model, require retraining, add solver stages, or increase the number of model evaluations. It only changes where the evaluations are placed. In this sense, SharpEuler moves adaptivity out of the inner sampling loop and into a cheap offline calibration stage; the deployed sampler remains a fixed-cost Euler sampler with exactly \(B\) evaluations.

SharpEuler is motivated by the local truncation error (LTE) of Euler integration along a flow trajectory. Euler's local truncation error is controlled by trajectory acceleration: regions where the learned velocity changes rapidly require smaller steps, while nearly linear regions can be traversed more coarsely. The same signal appears in backward error analysis \cite{hairer2006geometric} as the leading modified-equation bias. Thus, rather than placing steps uniformly or by hand, SharpEuler estimates where the trajectory acceleration is large and places
more evaluations in those regions.
This perspective also explains why the method is most useful at small sampling budgets. With many evaluations, even a uniform grid can resolve most changes in the trajectory. However, with only a few evaluations, a poorly placed Euler step can cross a sharp region without evaluating the intervening dynamics, which can mix modes or leave parts of the data distribution underrepresented. 
We observe this behavior both on synthetic distributions, where mode coverage and leakage can be measured directly, and in high-dimensional text-to-image sampling, where calibrated timesteps improve semantic fidelity.

\paragraph{Contributions.}
Our main contributions are as follows.
\begin{itemize}[leftmargin=*]
    \item We propose SharpEuler, a training-free, solver-aware method for
    constructing timestep schedules for pretrained flow matching models. It uses
    offline calibration trajectories to estimate a  sharpness
    profile and converts this profile into a fixed Euler grid for any inference
    budget.

    \item We provide three principles to justify SharpEuler. The \emph{numerical}
    principle shows, via local truncation error and backward error analysis, that  acceleration controls both the leading Euler error and the leading
    modified-equation bias. The \emph{variational} principle shows that natural
    timestep allocation objectives yield power-law  densities (Proposition~\ref{prop:informal_variational_power_law}), with the
    square-root case singled out by minimizing the leading Euler error proxy. The \emph{statistical} principle gives a stability guarantee (Theorem~\ref{thm:stat_guarantee}), showing that the calibrated schedule approaches the best population schedule in the admissible class as the number of calibration trajectories grows, up to smoothing bias.

    \item We demonstrate that SharpEuler improves low-budget sampling without
    increasing the number of model evaluations. On synthetic multimodal data, it
    reduces inter-mode leakage, improves mode coverage and sample quality, especially at very
    small Euler budgets. We then demonstrate its advantage on high-dimensional
    text-to-image generation tasks. 
\end{itemize}

\section{Background and Problem Setup}
\label{sec:background}

Flow matching models define a continuous transformation from a simple noise distribution to the data distribution through an ODE driven by a learned  velocity field~\cite{lipman2024flow}:
\begin{equation}
    \frac{dx(t)}{dt} =
    u_\theta(x(t),t),
    \qquad x(0) \sim p_0, \qquad x(t) \in \mathbb{R}^d, \qquad 
    t\in[0,1].
    \label{eq:learned_flow_ode_forward}
\end{equation}
We start from a base distribution $p_0 = \mathcal N(0,I)$ at $t=0$ and the ODE is integrated to $t=1$  to obtain new samples. Let $T(x_0):=x(1;x_0)$ and $\mu_1:=T_\#p_0$
denote the learned flow map and its output distribution.
Equivalently, we can use the reverse-time coordinate $s=1-t$ so that \(s=1\) corresponds to noise and \(s=0\) corresponds to data. If the learned velocity field is denoted by \(v_\theta(x,s)\), then the
forward-time velocity in Eq.~\eqref{eq:learned_flow_ode_forward} is $u_\theta(x,t) = -v_\theta(x,1-t)$.
Throughout, we assume the learned velocity field
is sufficiently smooth so that the displayed derivatives and Taylor expansions
are well defined along sampled trajectories.

%
%
%
In practice, the ODE in \eqref{eq:learned_flow_ode_forward} cannot be solved exactly and
must be discretized. Let $ 0=t_0<t_1<\cdots<t_B=1$
be an ascending forward-time schedule with $B$ intervals. 
The Euler update is
\begin{equation}
    x^{(k+1)}
    =
    x^{(k)}
    +
    (t_{k+1}-t_k)
    u_\theta(x^{(k)},t_k),
    \qquad
    k=0,\ldots,B-1.
    \label{eq:forward_euler_update}
\end{equation}
In the reverse-time implementation that we use, we consider the descending grid $s_k=1-t_k$, $1=s_0>s_1>\cdots>s_B=0,$
and Eq.~\eqref{eq:forward_euler_update} becomes
\begin{equation}
    x^{(k+1)}
    =
    x^{(k)}
    -
    (s_k-s_{k+1})
    v_\theta(x^{(k)},s_k),
    \qquad
    k=0,\ldots,B-1.
    \label{eq:reverse_euler_update}
\end{equation}

The sampler performs exactly one evaluation of the vector field per interval, and therefore uses exactly $B$
model evaluations. The integer $B$ is often small in fast generation settings, because each
evaluation of $v_{\theta}$ is a forward pass through a neural network.
Once the model, solver, and budget $B$ are fixed, the main remaining design
choice is the time grid itself. A uniform schedule sets
$s_k = 1-k/B$, while other schedules redistribute the same number of evaluations
toward selected parts of the trajectory. All such methods have the same online
cost when they use the same value of $B$; they differ only in where the model
evaluations are placed. This leads to the following problem setup.

\textbf{Problem setup.} Given a pretrained flow matching
model $v_{\theta}$ and a fixed evaluation budget $B$, construct a monotone
time grid $\mathcal{T}_B
    =
    \{s_k\}_{k=0}^B,$
$1=s_0>s_1>\cdots>s_B=0,$
that improves Euler sampling without changing the model, retraining, adding solver stages, or increasing the number of function evaluations (NFE). In the next section, we propose to choose this grid by measuring where the learned flow is
sharp, and to allocate the fixed Euler budget accordingly.

\section{SharpEuler: A Sharpness Aware Sampler}
\label{sec:method}

In this section, we introduce SharpEuler, our proposed  sharpness aware sampler. Given a pretrained
flow matching model and an evaluation budget \(B\), our goal is to construct a
deterministic time grid for Euler sampling that allocates more model evaluations
to regions where the learned dynamics are sharp, and fewer evaluations to
regions where they are smooth. Following Section~\ref{sec:background}, the
theory is written in forward time \(t\), with velocity \(u_\theta\), while the
implementation uses the descending reverse-time coordinate \(s=1-t\), with
velocity \(v_\theta\). 
The method has two stages:
\begin{itemize}[leftmargin=*]
    \item \textbf{Offline stage.} We estimate a sharpness profile by computing
    fine-grained Euler trajectories on a small calibration set and measuring how
    rapidly the learned vector field changes along these trajectories. This
    stage is performed only once for a given model and calibration distribution.

    \item \textbf{Online stage.} We convert this profile into a monotone grid
    and run Euler sampling on that grid, which requires the same number of model
    evaluations as a uniform Euler sampler with budget \(B\).
\end{itemize}

Detailed algorithms for the two stages are given in
App.~\ref{app:detailed_algo}. The motivation is numerical: Euler's method is
accurate when the trajectory is nearly linear over a step, but inaccurate when
the trajectory changes rapidly. Thus a fixed step size may be harmless in smooth
regions and too coarse in sharp ones, suggesting that the time grid should depend
on the learned dynamics.

For the reverse-time ODE, the leading local error is controlled by the trajectory acceleration
$\ddot{x}(t)
    =
    \frac{d}{dt}v_\theta(x(t),t).$
This derivative is taken along the sampled trajectory, so it reflects both explicit time variation in $v_\theta$ and variation induced as the state moves different velocity. Computing this quantity accurately is expensive for large neural networks, while classical adaptive error control would
require additional model evaluations during sampling. We therefore estimate sharpness offline using finite differences of consecutive velocity evaluations along fine-grained Euler trajectories. This proxy measures where the learned
trajectory changes direction or speed rapidly, and hence where a coarse Euler step is most likely to incur error. The resulting sharpness profile is model-, data-, and conditioning-distribution dependent. It is not an online error estimate for a particular sample, but an offline summary of where the velocity field of the pretrained model changes rapidly.

\subsection{Offline calibration}
\label{sec:offline_calibration}

The offline stage turns the numerical intuition above into a computable signal.
We would like to know, as a function of time, where the velocity field changes
rapidly along typical sampling trajectories. Since we want a fixed-cost sampler,
we estimate this information once from a calibration set and then reuse the
resulting profile during inference.
Let \(\mathcal{C}=\{c_j\}_{j=1}^M\) be a calibration set of representative
conditioning inputs, prompts, or samples. For unconditional models, we omit
\(c_j\) from the notation. We choose a fine-grained descending reverse-time grid
$1=s_0>s_1>\cdots>s_N=0,$
where $N$ is much larger than the evaluation budgets used at inference time.
For each calibration input \(c_j\), we run an \(N\)-step Euler trajectory and
record the reverse-time velocity evaluations
$V_{j,i}
    =
    v_{\theta}(x_{j,i},s_i;c_j),
    \ 
    i=0,\ldots,N-1.$
The reference trajectory itself is generated by
$x_{j,i+1}
    =
    x_{j,i}
    -
    (s_i-s_{i+1})
    V_{j,i}.$
These evaluations are already produced by the reference trajectory. We use their
successive differences to estimate how rapidly the learned vector field changes
along the path. For \(i=0,\ldots,N-2\), define
\begin{equation}
    \widehat{a}_{j,i}
    =
    \frac{
        V_{j,i+1}-V_{j,i}
    }{
        d_i
    }, \qquad   d_i
    =
    \frac12
    \left(
        |s_i-s_{i+1}|
        +
        |s_{i+1}-s_{i+2}|  \right) = \frac12 (s_i - s_{i+2}).
    \label{eq:finite_difference_acceleration}
\end{equation}
This is a finite difference approximation to the derivative of the velocity
along the trajectory. The sign depends on the direction in which time is
parameterized, but its norm does not. We therefore use
\(\|\widehat{a}_{j,i}\|\) as the local sharpness signal.
Averaging over the calibration set gives a global profile
$\bar{\mathcal{S}}_i
    =
    \frac{1}{M}
    \sum_{j=1}^{M}
    \left\|
        \widehat{a}_{j,i}
    \right\|,
    \ 
    i=0,\ldots,N-2$.
This averaging is important. We do not want to adapt the sampler separately for
each generated sample, since that would introduce online overhead and variable
cost. Instead, we estimate a reusable profile for the model and calibration
distribution. In this sense, the schedule is data informed but fixed at
inference time.

Motivated by variational principles (see Sec.~\ref{subsec:variational}),
we form the exponent-shaped sharpness signal
\begin{equation}
    I_{\gamma,i}
    =
    \left(
        \bar{\mathcal{S}}_i+\varepsilon_a
    \right)^\gamma,
    \qquad
    \gamma>0,
    \qquad
    i=0,\ldots,N-2,
    \label{eq:sharpness_power_signal}
\end{equation}
where \(\varepsilon_a\geq0\) is a small numerical floor. The raw profile can be
noisy, especially for small calibration sets. We therefore smooth it with a
Gaussian kernel of bandwidth \(\sigma\), i.e.
$\widetilde I_{\gamma}
    =
    \mathrm{Smooth}_{\sigma}
    \left( I_\gamma\right)$.
We then normalize the smoothed values
into discrete masses
$\widehat w_i
    =
    \frac{
        \widetilde I_{\gamma,i}
    }{
        \sum_{r=0}^{N-2}
        \widetilde I_{\gamma,r}
    },
    \ 
    i=0,\ldots,N-2.$
The support of this finite-difference signal is naturally centered at
$\widetilde s_i
    =
    \frac{s_i+s_{i+2}}{2}, \ 
    i=0,\ldots,N-2.$
To be consistent with the forward-time  notation, we
set $\widetilde t_i =  1-\widetilde s_i.$
The pairs \((\widetilde t_i,\widehat w_i)\) are sorted in ascending
\(\widetilde t_i\), an empirical CDF is formed with a left endpoint anchor at
\(t=0\), and the inverse CDF is evaluated by linear interpolation. This produces
a deterministic schedule for any evaluation budget.

\subsection{Online sampling}

At inference time, the sampler is ordinary Euler integration on the calibrated
grid. Let \(F_{\rm emp}^{-1}\) denote the linearly interpolated inverse-CDF map
obtained from the offline stage. For a budget \(B\), we compute ascending
forward-time quantiles
$\widetilde t_b
    =
    F_{\rm emp}^{-1}\left(\frac{b}{B}\right),
    \ 
    b=0,\ldots,B,$
and convert them to a descending reverse-time  schedule
$s_b = 1-\widetilde t_b,
    \ 
    b=0,\ldots,B.$
Given a conditioning input \(c\) and an initial noise sample
\(x^{(0)}\sim\mathcal{N}(0,I)\), we iterate
$x^{(k+1)}
    =
    x^{(k)}
    -
    (s_k-s_{k+1})
    v_{\theta}(x^{(k)},s_k;c),
    \ 
    k=0,\ldots,B-1.$
The only difference from uniform  sampling is the choice of the time grid
\(\{s_k\}\). The model is unchanged, no gradients are computed, and no online
error estimates or rejected steps are used. The number of model evaluations is
exactly \(B\). The calibrated schedule can be viewed as a time
reparameterization \cite{stancevic2025entropic}: the empirical sharpness profile defines a density over
forward model time, and the sampler uses uniform quantiles in the corresponding
CDF.

This distinction from classical adaptive solvers is important. Classical
adaptive methods choose step sizes during integration and usually spend
additional function evaluations to estimate or control local error. Sharpness
aware sampling performs this adaptation once, offline, using representative
trajectories. The resulting schedule is fixed before inference begins. Thus, the
online cost remains deterministic and identical to that of a uniform Euler
sampler with the same budget.

\section{Why SharpEuler: Three Justifying Principles}
\label{sec:theory_sharpeuler}


Throughout this section, we use the forward-time convention for clarity of analysis, and assume that the learned velocity field
\(u_{\theta}:\mathbb R^d\times[0,1]\to\mathbb R^d\) is sufficiently smooth
on the region visited by sampled trajectories so that the displayed derivatives and Taylor expansions
are well defined.

Let \(x(t;x_0)\) solve the
forward-time ODE \eqref{eq:learned_flow_ode_forward}
and define the  acceleration sharpness $a(t;x_0)=\|\ddot x(t;x_0)\|$.
We justify SharpEuler through three 
principles.

\subsection{Numerical principle: acceleration controls Euler error and modified-equation bias} \label{subsec:lte}

\begin{tcolorbox}[
  colback=parchment,
  colframe=sandborder,
  boxrule=0.3pt,
  left=8pt,
  right=8pt,
  top=1pt,
  bottom=1pt
]
\centering
\emph{
Principle I: The acceleration magnitude 
\(\|\ddot x(t)\|\) is the quantity selected by both Euler local truncation error
and backward error analysis.
}
\end{tcolorbox}

The first justification comes from local truncation error. 

\begin{proposition}[Informal: Euler error selects acceleration]
\label{prop:informal_euler_acceleration}
For the forward-time ODE \eqref{eq:learned_flow_ode_forward}, the leading local
truncation error of one Euler step from the current state \(x(t)\) is proportional
to \(h^2\ddot x(t)\). Hence, to leading order, the size of the one-step Euler
error is controlled by \(h^2\|\ddot x(t)\|\).
\end{proposition}

See Proposition \ref{prop_detailedETA} for a detailed version. Indeed, Taylor expansion gives
\begin{equation}
    x(t+h)
    -
    \left[
        x(t)+h u_{\theta}(x(t),t)
    \right]
    =
    \frac{h^2}{2}\ddot x(t)
    +
    \mathcal O(h^3)
    \label{eq:theory_euler_lte}
\end{equation}
as $h \to 0$, showing that the leading local error is $h^2 \|\ddot{x}(t)\|/2$. Moreover, using the chain rule,
$\ddot x(t)
    =
    \frac{d}{dt} u_{\theta}(x(t),t)
    =
    \partial_t u_{\theta}(x(t),t)
    +
    \nabla_x u_{\theta}(x(t),t) u_{\theta}(x(t),t).$
Therefore, finite differences of learned velocities along calibration
trajectories provide a first-order Euler local-error proxy for \(a(t;x_0)\), i.e.
$\frac{
        u_{\theta}(x(t_{k+1}),t_{k+1})
        -
        u_{\theta}(x(t_k),t_k)
    }{\Delta t}
    =
    \ddot x(t_k)
    +
    \mathcal O(\Delta t).$

The same quantity is selected by backward error analysis \cite{hairer2006geometric}. Rather than viewing
Euler only as an approximation to \eqref{eq:learned_flow_ode_forward}, backward error
analysis asks which nearby, step-size-dependent ODE is implicitly solved by the
Euler map. For
$\psi_h(t,x)=x+h u_{\theta}(x,t)$,
one seeks a formal modified vector field, 
$u_{\theta,h}(x,t)
    =
    u_{\theta}(x,t)
    +
    h u_{\theta,1}(x,t)
    +
    \mathcal O(h^2)$, 
whose exact time-\(h\) flow agrees with \(\psi_h\) up to local error
\(\mathcal O(h^3)\). Matching Taylor coefficients gives
\begin{equation}
    u_{\theta,h}(x,t)
    =
    u_{\theta}(x,t)
    -
    \frac{h}{2}
    \left(
        \partial_t u_{\theta}(x,t)
        +
        \nabla_x u_{\theta}(x,t)u_{\theta}(x,t)
    \right)
    +
    \mathcal O(h^2).
    \label{eq:theory_bea_modified_field}
\end{equation}
Along exact trajectories, the leading perturbation in the modified vector field
is therefore $-\frac{h}{2}\ddot x(t)$.
Consequently, regions with large acceleration are, to leading order, precisely
the regions where Euler is both locally less accurate and more strongly biased
toward a modified dynamics. This motivates using acceleration as a
solver-aware timestep allocation signal. Full derivations are  in
App.~\ref{app:fulltheory}.

\subsection{Variational principle: power-law densities arise from optimal allocation} \label{subsec:variational}

\begin{tcolorbox}[
  colback=parchment,
  colframe=sandborder,
  boxrule=0.3pt,
  left=8pt,
  right=8pt,
  top=1pt,
  bottom=1pt
]
\centering
\emph{
Principle II: Once a positive importance profile \(I(t)\) is specified,
natural variational allocation rules produce power-law timestep densities.
}
\end{tcolorbox}

The second justification is variational. 
A timestep density $\rho(t)$ specifies how a fixed evaluation budget is distributed over sampling time. If
$\rho(t)$ is small in a region where the importance profile $I(t)$ is large, then the sampler under-resolves an important part of the path. Thus a natural
allocation rule should penalize assigning too little timestep density to high-importance regions.

The Euler local-error proxy gives one concrete example. Since a timestep density
induces a local step size
$h(t)\approx \frac{1}{B\rho(t)},$
and  local truncation error scales quadratically in $h(t)$, the leading
accumulated Euler-risk contains a term of the form
$\int_0^1
    \frac{I(t)}{\rho(t)}
    \,dt.$
This corresponds to a power-law penalty against small $\rho(t)$. At the other
extreme, if the goal is simply to match a prescribed importance profile, then the weighted cross-entropy objective
$\mathcal{C}(\rho; I) :=  -
    \int_0^1
    I(t)\log \rho(t)
    \,dt$
is natural. This motivates the following family, which interpolates between these two allocation principles.

Let \(I:[0,1]\to(0,\infty)\) be a positive importance profile and let \(\rho\)
be a normalized timestep density. For \(\beta\geq 0\), define
\begin{equation}
    \psi_\beta(r)
    =
    \begin{cases}
    \dfrac{r^{-\beta}-1}{\beta}, & \beta>0,\\[1.2ex]
    -\log r, & \beta=0,
    \end{cases}
    \qquad r>0.
\end{equation}
The objective
\begin{equation}
    \min_{\rho>0,\ \int_0^1\rho(t)\,dt=1}
    \mathcal F_\beta(\rho;I),
    \qquad
    \mathcal F_\beta(\rho;I)
    =
    \int_0^1 I(t)\psi_\beta(\rho(t))\,dt
    \label{eq:theory_variational_family}
\end{equation}
penalizes under-allocation of density in high-importance regions. The parameter
\(\beta\) controls the strength of this penalty: \(\beta=0\) gives the
logarithmic cross-entropy penalty, while \(\beta=1\) gives the Euler-risk
penalty \(\int I(t)/\rho(t)\,dt\), up to an additive constant and scale.

\begin{proposition}[Variational allocation gives a power law]
\label{prop:informal_variational_power_law}
The unique minimizer of \eqref{eq:theory_variational_family} is
\begin{equation}
    \rho_\beta^\star(t)
    =
    \frac{
        I(t)^{1/(\beta+1)}
    }{
        \int_0^1 I(s)^{1/(\beta+1)}\,ds
    }.
    \label{eq:theory_variational_solution}
\end{equation}
Thus variational timestep allocation naturally produces power-law densities in
the importance profile.
\end{proposition}

See Proposition \ref{prop:unified_variational_density} for a detailed version. The solution
$\rho_\beta^\star(t)
    \propto
    I(t)^{1/(\beta+1)}$
shows how the variational family converts an importance profile into a timestep
density. When \(\beta=0\), the scheduler directly matches the importance
profile, \(\rho_0^\star(t)\propto I(t)\). When \(\beta=1\), it produces the
square-root law, \(\rho_1^\star(t)\propto \sqrt{I(t)}\). Larger \(\beta\)
produces flatter densities because the power \(1/(\beta+1)\) becomes smaller.

This family recovers two timestep schedules of SharpEuler.
First, let
$\bar a(t)=\mathbb E_{x_0\sim p_0}\|\ddot x(t;x_0)\|$ be the population acceleration sharpness. If \(\rho\) induces local
step size
$h(t)\approx \frac{1}{B\rho(t)}$,
then the accumulated leading Euler local-error (or risk) proxy scales as $\mathcal E[\rho]
    \approx
    \frac{1}{2B}
    \int_0^1
    \frac{\bar a(t)}{\rho(t)}
    \,dt$.
This corresponds to \(\beta=1\) and \(I(t)=\bar a(t)\), yielding
 $\rho_{\rm Euler}^\star(t)
    \propto
    \bar a(t)^{1/2}$.

Second, the profile-matching objective corresponds to \(\beta=0\), i.e.
$\min_{\rho>0,\ \int\rho=1}
  \mathcal{C}(\rho; I).$
Its unique minimizer is
$\rho_I(t)
    =
    \frac{I(t)}{\int_0^1 I(s)\,ds}$.
Equivalently, with
$\pi_I(t)=\frac{I(t)}{\int_0^1 I(s)\,ds}$,
the objective equals a constant multiple of
\(H(\pi_I)+{\rm KL}(\pi_I\|\rho)\), and is minimized uniquely by
\(\rho=\pi_I\).

\begin{tcolorbox}[
  colback=parchment,
  colframe=sandborder,
  boxrule=0.3pt,
  left=8pt,
  right=8pt,
  top=1pt,
  bottom=1pt
]
\centering
\emph{
SharpEuler uses the solver-aware profile
\(I_\gamma(t)=\bar a(t)^\gamma\), hence
\(\rho_\gamma(t)\propto \bar a(t)^\gamma\).
The special case \(\gamma=1/2\) is selected by the Euler-risk proxy.
}
\end{tcolorbox}

SharpEuler with exponent \(\gamma\) can therefore be interpreted as
profile-matching to the solver-aware importance profile
\begin{equation}
    I_\gamma(t)
    =
    \bar a(t)^\gamma,
    \qquad
    \rho_\gamma(t)
    \propto
    \bar a(t)^\gamma.
    \label{eq:theory_sharpeuler_gamma_density}
\end{equation}
Here \(\gamma\) shapes the sharpness profile, while the variational parameter
\(\beta\) describes how a chosen profile is converted into a density. The
implemented exponent schedules correspond to the profile-matching case
\(\beta=0\). The square-root case \(\gamma=1/2\) is special: up to small smoothing, it coincides with the Euler-risk minimizer
\(\rho(t)\propto \bar a(t)^{1/2}\). For \(\gamma\neq1/2\), the schedule remains
a valid profile-matching density, but is not generally Euler-risk optimal. In
all cases, the underlying sharpness signal is justified by the numerical principle. Full derivations and details are given in App.~\ref{app:fulltheory}.

\subsection{Statistical principle: statistical stability of the  calibrated schedule}

\begin{tcolorbox}[
  colback=parchment,
  colframe=sandborder,
  boxrule=0.3pt,
  left=8pt,
  right=8pt,
  top=1pt,
  bottom=1pt
]
\centering
\emph{Principle III:
With \(M\) calibration trajectories, the
calibrated schedule is near its population oracle up to a finite-calibration
term of order $\mathcal O (
\sqrt{\log(N/\delta)/M})$
on a grid of size \(N\).
}
\end{tcolorbox}

The preceding principles are population-level statements that involve the
population sharpness profile
$\bar a(t)
    =
    \mathbb E_{x_0\sim p_0}\|\ddot x(t;x_0)\|.$
In practice, SharpEuler estimates this profile from finitely many calibration
trajectories and then constructs an empirical timestep density. Thus we also need the resulting grid to be statistically stable. The theorem below provides such a guarantee under reasonable assumptions.

\begin{theorem}[Informal: statistical stability for the square-root schedule $\gamma=1/2$]
\label{thm:informal_statistical_stability}
Let \(\widehat w\) be the  schedule calibrated from \(M\) independent
trajectories on a grid with \(N\) intervals, and let \(w^\star\) be the
population oracle schedule for the Euler local-error proxy \(\mathcal J\) defined in App.~\ref{sec:stat_guarantee}. Under  boundedness, concentration, and Euler stability assumptions, with probability
at least \(1-\delta\),
\begin{equation}
    W_1(\mu_{\widehat w},\mu_1)
    \leq
    \frac{1}{B}
    \left[
        C_{\rm Eul}\mathcal J(w^\star)
        +
        C_1
        \sqrt{\frac{\log(N/\delta)}{M}}
        +
        C_2
    \right].
    \label{eq:informal_statistical_stability}
\end{equation}
Here $W_1$ denotes the 1-Wasserstein distance, \(\mu_{\widehat w}\) is the terminal law of the
calibrated Euler sampler, \(\mu_1\) is the terminal law of the exact learned
ODE flow, $C_{\mathrm{Eul}}$ is the Euler stability constant, \(C_1\) collects concentration and schedule-stability constants, and
\(C_2\) collects smoothing, finite-grid, and finite-difference proxy
bias terms.
\end{theorem}

See Theorem \ref{thm:stat_guarantee} for a detailed version and discussions for general $\gamma$ cases. The bound separates the numerical sampling error into an oracle Euler term, a finite-calibration term of order $\sqrt{\log(N/\delta)/M}$, and a residual implementation-bias term. Thus
SharpEuler preserves the $\mathcal O(1/B)$ Euler scaling while its calibrated schedule
concentrates toward its population counterpart as $M$ grows.  The result controls numerical error relative to the  learned ODE flow, not model error relative to the data distribution.

\begin{figure}[!b]
    \centering
    \includegraphics[width=0.8\textwidth]{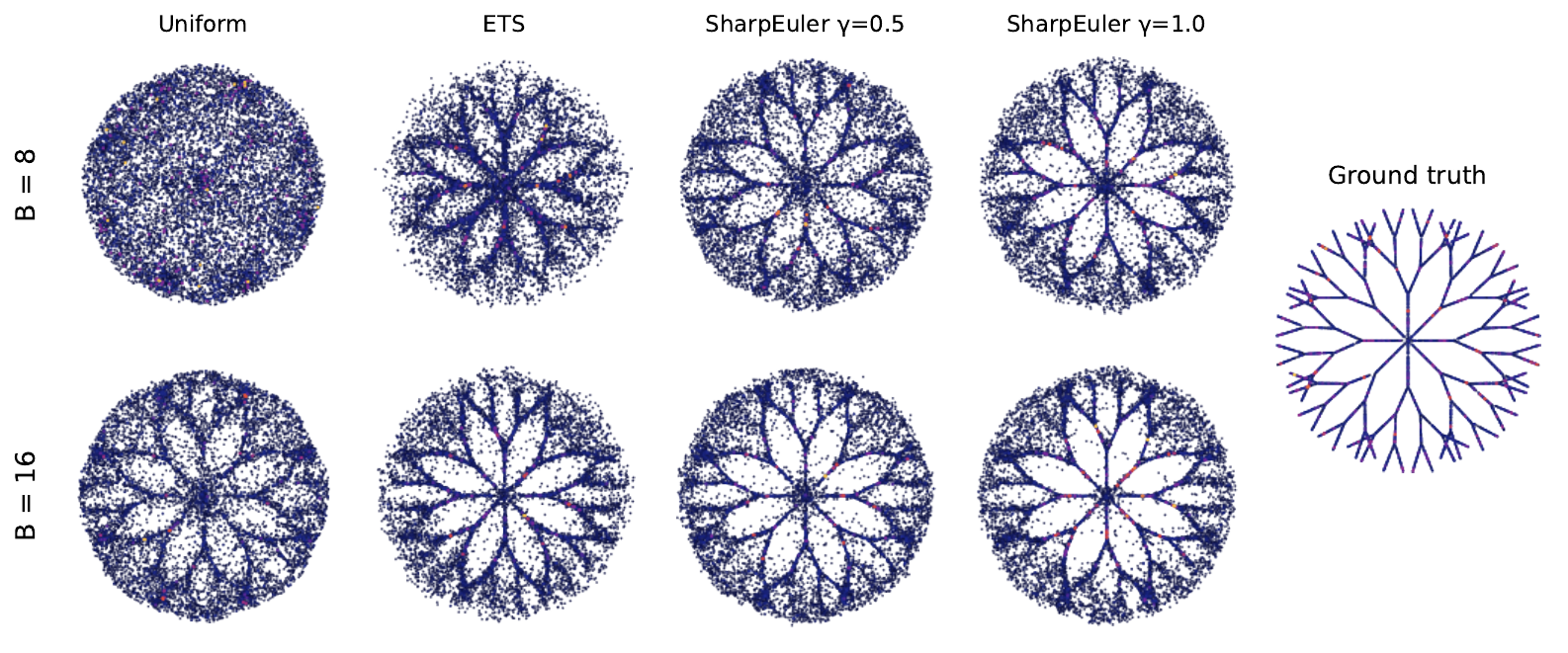}
    \includegraphics[width=0.8\textwidth]{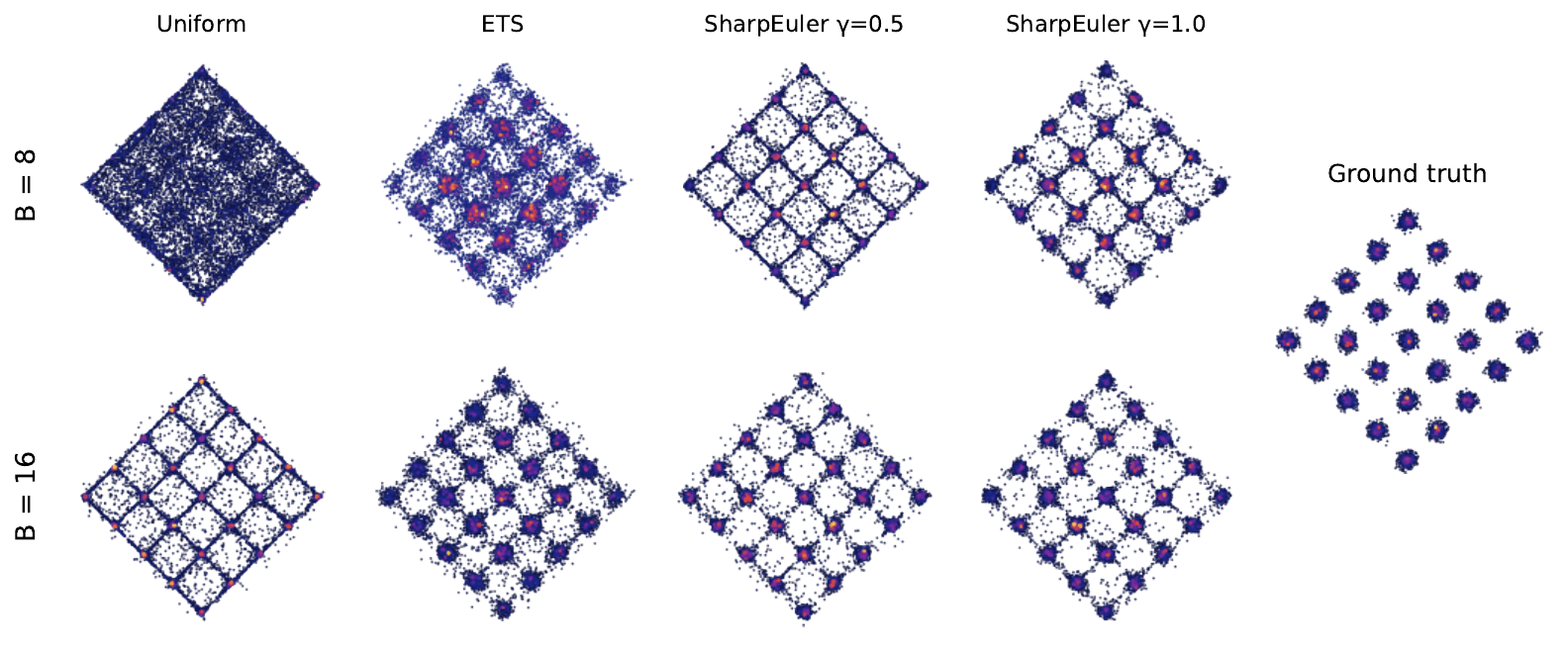}
    \caption{Generated samples on Branched Manifold (top) and Rotated Grid (bottom) at \(B\in\{8,16\}\), comparing Uniform, ETS, and SharpEuler with \(\gamma\in\{0.5,1.0\}\). At fixed NFE, SharpEuler better matches the visible support of the target distributions in the low-budget regime. The two exponents
produce different sample patterns: \(\gamma=0.5\) favors coverage of the target support, while \(\gamma=1.0\) produces samples that are more tightly concentrated around the visible manifold or grid structure.}
    \label{fig:toy_samples}
\end{figure}

\section{Experiments}
\label{sec:experiments}

We evaluate whether SharpEuler improves fixed-budget Euler sampling by changing where, but not how many, model evaluations are placed. We report budgets in
number of function evaluations (NFE), where one NFE corresponds to one evaluation of the learned vector field. The experiments address two questions. First, on controlled two-dimensional distributions, does calibrated timestep placement improve sample fidelity and coverage at small NFE? Second, on a large pretrained text-to-image model, does the same offline calibration procedure improve low-NFE sampling?
Whenever possible,
the model, initial noise, conditioning input, and solver are held fixed, so
differences are attributable only to the timestep grid. Implementation details
and additional results are given in App.~\ref{app:exp_details} and~\ref{app:results}.

\subsection{Synthetic two-dimensional experiments}
\label{sec:exp_synthetic}

We first evaluate SharpEuler in a controlled setting where geometric errors are
visible and distributional discrepancies can be measured directly. We use two two-dimensional target distributions, \emph{Branched Manifold} and \emph{Rotated Grid}. The first tests whether a sampler resolves thin, curved structure; the second tests whether it preserves separated modes without
introducing off-manifold mass.

Figure~\ref{fig:toy_samples} shows generated samples at \(B=8\) and \(B=16\). At \(B=8\), Uniform and ETS are visibly under-resolved. On the Branched Manifold, the generated samples miss parts of the fine branch structure. On the Rotated Grid, samples are less concentrated around the target modes and include
more mass between modes. SharpEuler produces samples that better match the visible support of the target distribution at the same NFE. The two exponents produce different sample patterns: \(\gamma=0.5\) gives better coverage of the target support, while \(\gamma=1.0\) produces samples that are more tightly
concentrated around the visible manifold or grid structure. At \(B=16\), the gap between methods narrows, but SharpEuler remains visually closer to the target samples than Uniform and ETS.

Table~\ref{tab:toy_metrics} quantifies the same behavior using Density, Coverage~\citep{naeem2020reliable}, and squared 2-Wasserstein distance \(W_2^2\). SharpEuler gives the strongest results in the low-budget regime, where timestep placement matters most. At \(B=8\), SharpEuler substantially improves Coverage on both datasets and achieves the best Density on the Rotated Grid, while \(\gamma=1.5\) gives the highest Density on the Branched Manifold. The \(W_2^2\) metric is consistently minimized by \(\gamma=0.5\), in agreement with the LTE-based choice: this exponent gives the best overall distributional
match rather than the most concentrated samples.
The Density and Coverage columns also show the effect of the exponent. On the Branched Manifold, increasing \(\gamma\) raises Density but can reduce Coverage and worsen \(W_2^2\), indicating that the samples become more concentrated on parts of the target support. On the Rotated Grid, \(\gamma=1.0\) gives the best
Coverage across all budgets and the best or near-best Density, while \(\gamma=1.5\) again loses coverage. These results suggest using \(\gamma=0.5\) when the goal is distributional agreement, and \(\gamma=1.0\) when the priority is a stronger fidelity--coverage tradeoff. 

\begin{table}[t]
\centering
\caption{Quantitative results on the two synthetic distributions. Density (D) measures sample fidelity, Coverage (C) measures support coverage, and \(W_2^2\) measures distributional discrepancy. Higher is better for D and C;
lower is better for \(W_2^2\). Best values per metric and budget are shown in \textbf{bold}.}
\label{tab:toy_metrics}
\vspace{+0.2cm}
\resizebox{\textwidth}{!}{%
\setlength{\tabcolsep}{4pt}
\begin{tabular}{ll ccc ccc ccc ccc}
\toprule
& & \multicolumn{3}{c}{$B=8$} & \multicolumn{3}{c}{$B=12$} & \multicolumn{3}{c}{$B=16$} & \multicolumn{3}{c}{$B=20$} \\
\cmidrule(lr){3-5} \cmidrule(lr){6-8} \cmidrule(lr){9-11} \cmidrule(lr){12-14}
Dataset & Method & D$\uparrow$ & C$\uparrow$ & $W_2^2\downarrow$ & D$\uparrow$ & C$\uparrow$ & $W_2^2\downarrow$ & D$\uparrow$ & C$\uparrow$ & $W_2^2\downarrow$ & D$\uparrow$ & C$\uparrow$ & $W_2^2\downarrow$ \\
\midrule
\multirow{5}{*}{\shortstack[l]{Branched\\Manifold}}
& Uniform                  & 0.139 & 0.357 & 0.020 & 0.163 & 0.398 & 0.016 & 0.214 & 0.456 & 0.013 & 0.265 & 0.512 & 0.013 \\
& ETS                      & 0.269 & 0.428 & 0.055 & 0.316 & 0.520 & 0.026 & 0.359 & 0.551 & 0.017 & 0.364 & 0.575 & 0.016 \\
& SharpEuler $\gamma{=}0.5$ & 0.261 & 0.479 & \textbf{0.017} & 0.338 & 0.554 & \textbf{0.013} & 0.385 & 0.597 & \textbf{0.012} & 0.404 & 0.604 & \textbf{0.012} \\
& SharpEuler $\gamma{=}1.0$ & 0.422 & \textbf{0.582} & 0.018 & 0.435 & \textbf{0.598} & 0.014 & 0.441 & \textbf{0.620} & 0.015 & 0.425 & \textbf{0.613} & 0.013 \\
& SharpEuler $\gamma{=}1.5$ & \textbf{0.510} & 0.554 & 0.035 & \textbf{0.474} & 0.576 & 0.020 & \textbf{0.458} & 0.608 & 0.018 & \textbf{0.456} & 0.603 & 0.015 \\
\midrule
\multirow{5}{*}{\shortstack[l]{Rotated\\Grid}}
& Uniform                  & 0.426 & 0.581 & 0.022 & 0.705 & 0.771 & 0.015 & 0.810 & 0.685 & 0.012 & 0.855 & 0.699 & 0.010 \\
& ETS                      & 0.652 & 0.723 & 0.038 & 0.779 & 0.831 & 0.026 & 0.846 & 0.877 & 0.017 & 0.885 & 0.902 & 0.017 \\
& SharpEuler $\gamma{=}0.5$ & 0.845 & 0.749 & \textbf{0.020} & 0.900 & 0.890 & \textbf{0.013} & 0.924 & 0.921 & \textbf{0.011} & 0.941 & 0.928 & \textbf{0.009} \\
& SharpEuler $\gamma{=}1.0$ & \textbf{0.913} & \textbf{0.872} & 0.029 & 0.931 & \textbf{0.924} & 0.020 & 0.942 & \textbf{0.934} & 0.018 & \textbf{0.953} & \textbf{0.939} & 0.015 \\
& SharpEuler $\gamma{=}1.5$ & 0.903 & 0.700 & 0.070 & \textbf{0.940} & 0.820 & 0.044 & \textbf{0.951} & 0.876 & 0.033 & 0.948 & 0.904 & 0.028 \\
\bottomrule
\end{tabular}
}
\end{table}

\begin{table}[t]
\caption{FLUX.1-dev results at \(1024{\times}1024\) resolution on 80 held-out prompts. RMSE measures pixel-level deviation from the 50-step pipeline reference, FID measures distributional discrepancy from the reference image set, and CLIP measures prompt-image alignment. Higher is better for CLIP; lower is better for RMSE and FID. Best values per metric and budget are shown in \textbf{bold}.}
\label{tab:flux}
\small
\vspace{+0.2cm}
\setlength{\tabcolsep}{4pt}
\resizebox{\textwidth}{!}{%
\begin{tabular}{l ccc ccc ccc ccc}
\toprule
& \multicolumn{3}{c}{$B=8$} & \multicolumn{3}{c}{$B=12$} & \multicolumn{3}{c}{$B=16$} & \multicolumn{3}{c}{$B=20$} \\
\cmidrule(lr){2-4} \cmidrule(lr){5-7} \cmidrule(lr){8-10} \cmidrule(lr){11-13}
Method & RMSE$\downarrow$ & CLIP$\uparrow$ & FID$\downarrow$ & RMSE$\downarrow$ & CLIP$\uparrow$ & FID$\downarrow$ & RMSE$\downarrow$ & CLIP$\uparrow$ & FID$\downarrow$ & RMSE$\downarrow$ & CLIP$\uparrow$ & FID$\downarrow$ \\
\midrule
Uniform                  & 44.38 & 0.32 & 124.31 & 39.97 & 0.32 & 104.77 & 35.95 & \textbf{0.33} & 90.32 & 28.58 & \textbf{0.33} & 74.20 \\
Shifted $\alpha{=}3$      & 43.76 & 0.32 & 120.84 & 39.70 & 0.32 & 104.01 & 35.96 & \textbf{0.33} & 89.35 & 28.55 & \textbf{0.33} & 74.05 \\
SharpEuler $\gamma{=}0.5$ & \textbf{24.63} & \textbf{0.33} & \textbf{80.61} & \textbf{16.51} & \textbf{0.33} & \textbf{48.65} & \textbf{12.18} & \textbf{0.33} & \textbf{32.76} & \textbf{12.40} & \textbf{0.33} & \textbf{28.56} \\
\bottomrule
\end{tabular}%
}
\end{table}

\subsection{Text-to-image generation with FLUX}

We next evaluate whether the same calibration procedure improves sampling for a large pretrained rectified flow model in a high-dimensional vision setting. We use FLUX.1-dev~\citep{flux2024}, a 12B-parameter rectified flow transformer for
text-to-image generation, at its native \(1024\times1024\) resolution. We calibrate SharpEuler once using 64 prompts and a 50-step reference grid, which matches the default inference budget for FLUX.1-dev in \texttt{diffusers}. The calibration prompts span landscapes, urban scenes, wildlife, and abstract
content. The resulting sharpness profile is then fixed and reused for 80 held-out prompts with no prompt overlap. All samples are generated in \texttt{bfloat16} with classifier-free guidance scale \(3.5\).
We compare SharpEuler with \(\gamma=0.5\) against two profile-agnostic schedules: the default FLUX pipeline schedule and a static shifted schedule with \(\alpha=3\), following the SD3-style FlowMatch Euler shift~\citep{esser2024scaling}.
All methods use the same Euler update, the same evaluation budgets \(B\in\{8,12,16,20\}\), and the same number of model evaluations. Thus, as in the synthetic experiments, the only difference is the timestep grid. We omit ETS at this scale because its flow matching adaptation requires Jacobian-vector
products through the transformer which are too expensive to obtain.

\begin{figure}[!t]
    \centering
    \includegraphics[width=0.95\textwidth]{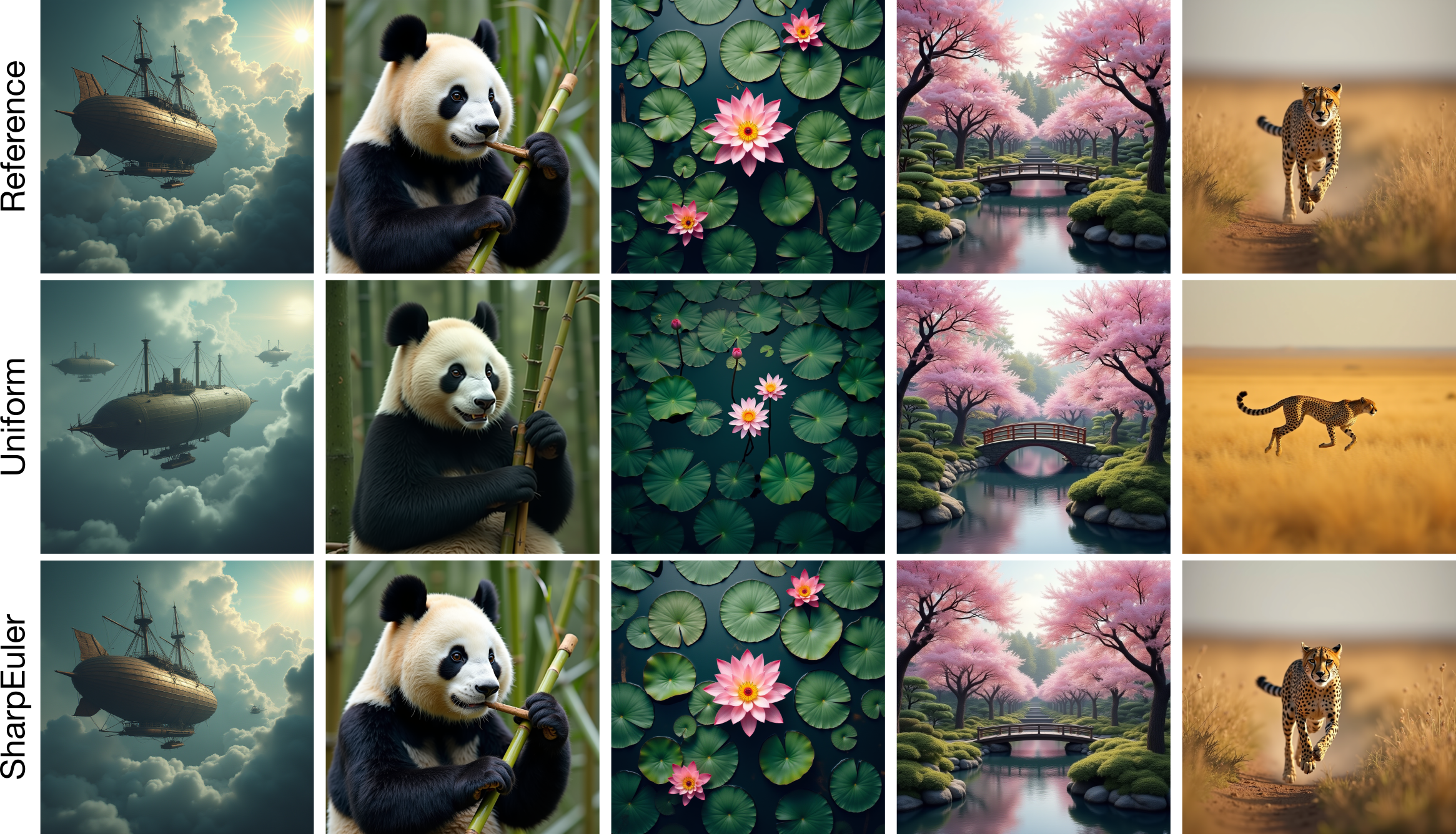}
    \caption{Representative FLUX.1-dev generations at \(B=16\), comparing the default uniform pipeline schedule and SharpEuler against the corresponding 50-step reference images. All methods use the same Euler sampler, prompt conditioning, and number of model evaluations. Relative to the pipeline schedule, SharpEuler more faithfully preserves scene composition, object structure, lighting, and fine visual details from the reference images. These  differences are consistent with results in Table~\ref{tab:flux_judge}.}
    \label{fig:flux_visual}
\end{figure}

\begin{table}[!t]
\centering
\caption{GPT-5.5 vision-judge comparison between SharpEuler
\(\gamma=0.5\) and the default FLUX pipeline schedule. For each prompt, the
judge compares both images against the same 50-step reference image and text
prompt. Higher is better for scores, \(\Delta\), and pairwise wins.}
\label{tab:flux_judge}
\small
\setlength{\tabcolsep}{5pt}
\begin{tabular}{l c c c c c c}
\toprule
Budget &
SharpEuler score &
Pipeline score &
Mean \(\Delta\uparrow\) &
Std.\ \(\Delta\) &
\(p\)-value &
SharpEuler wins \\
\midrule
\(B=8\) & \textbf{7.98} & 6.02 & \textbf{1.96} & 0.88 & \(8.2{\times}10^{-21}\) & \textbf{50/50} \\
\(B=12\) & \textbf{8.74} & 6.50 & \textbf{2.24} & 0.82 & \(1.6{\times}10^{-24}\) & \textbf{50/50} \\
\(B=16\) & \textbf{9.04} & 6.82 & \textbf{2.22} & 0.93 & \(5.1{\times}10^{-22}\) & \textbf{50/50} \\
\(B=20\) & \textbf{9.14} & 7.32 & \textbf{1.82} & 0.83 & \(1.2{\times}10^{-20}\) & \textbf{50/50} \\
\(B=24\) & \textbf{9.35} & 7.52 & \textbf{1.83} & 0.89 & \(9.8{\times}10^{-12}\) & \textbf{50/50} \\
\bottomrule
\end{tabular}
\end{table}

Table~\ref{tab:flux} reports three complementary measures against the 50-step pipeline reference. RMSE measures pixel-level agreement with the reference sample for the same prompt and seed. FID measures distributional agreement between the generated image set and the reference set. CLIP score measures
prompt-image semantic alignment. SharpEuler achieves the lowest RMSE and FID at every budget, with the largest gains in the low-NFE regime. For example, at \(B=12\), RMSE decreases from \(39.97\) to \(16.51\), and FID decreases from \(104.77\) to \(48.65\). At \(B=16\), FID decreases from \(90.32\) to \(32.76\). The shifted \(\alpha=3\) schedule remains close to the default uniform pipeline schedule across budgets, which indicates that the improvement is not explained by a generic monotone shift. CLIP scores stay within a narrow range across all
methods. SharpEuler therefore improves agreement with the high-step reference while preserving prompt alignment at fixed NFE.

We also evaluate perceptual agreement with the high-step reference using
GPT-5.5 as a vision judge. For each prompt, the judge receives the 50-step
reference image, the SharpEuler image, the Pipeline image, and the text prompt,
and returns both calibrated scores and a forced pairwise preference. Candidate
order is randomized to avoid position bias; the judge is instructed to compare
semantic content, composition, style, lighting, color, and salient spatial
relationships rather than pixel identity. At \(B=16\), SharpEuler receives a
mean score of \(9.04\), compared with \(6.82\) for Pipeline, giving a mean paired
improvement of \(2.22\) points. The judge selects SharpEuler in all 50 pairwise
comparisons. This result is consistent with the RMSE and FID improvements in
Table~\ref{tab:flux}: the calibrated grid produces images that are not only
closer to the reference under automatic metrics, but also closer under a
prompt-conditioned visual comparison.
Figure~\ref{fig:flux_visual} shows representative FLUX generations at \(B=16\),
comparing SharpEuler and the default pipeline schedule against the 50-step
reference image.

\section{Conclusion}

In this work, we asked where a flow matching sampler should spend its finite Euler budget. SharpEuler answers this question by calibrating a timestep schedule from the learned dynamics of a pretrained model. It estimates where the velocity field changes rapidly along calibration trajectories, smooths the
resulting sharpness profile, and converts it into a deterministic Euler grid for any inference budget. The deployed sampler remains unchanged except for the time grid: it uses the same model, the same Euler update, and the same NFE.
We justify SharpEuler with numerical, variational, and statistical principles. Experiments on synthetic distributions and FLUX.1-dev show that this calibration improves low-budget sampling at fixed NFE, reducing mode leakage, improving coverage, and producing images closer to high-step references while preserving prompt alignment.
These results suggest that efficient flow sampling depends not only on the solver, but also on how its evaluations are placed. Even without retraining or additional solver stages, a fixed-cost Euler sampler can improve when its steps are allocated according to how rapidly the learnt velocity field changes.


\subsection*{Acknowledgments}
NBE would like to acknowledge support from the U.S. Department of Energy, Office of Science, Office of Advanced Scientific Computing Research, EXPRESS: 2025 Exploratory Research for Extreme-Scale Science program, under Contract Number DE-AC02-05CH11231 at Berkeley Lab. The views, opinions, and/or findings expressed are those of the authors and should not be interpreted as representing the official views or policies of the Department of Defense or the U.S. Government.
SHL would like to acknowledge support from the Wallenberg Initiative on Networks and Quantum Information (WINQ) and the Swedish Research Council (VR/2021-03648).

\bibliographystyle{plain}
\bibliography{references}

@article{albergo2023stochastic,
  title={Stochastic interpolants: A unifying framework for flows and diffusions},
  author={Albergo, Michael S and Boffi, Nicholas M and Vanden-Eijnden, Eric},
  journal={arXiv preprint arXiv:2303.08797},
  year={2023}
}

@inproceedings{stancevic2025entropic,
  title={Entropic Time Schedulers for Generative Diffusion Models},
  author={Stancevic, Dejan and Handke, Florian and Ambrogioni, Luca},
  booktitle={Advances in Neural Information Processing Systems 38 (NeurIPS 2025)},
  year={2025},
}

@inproceedings{lee2023minimizing,
  title={Minimizing trajectory curvature of {O}{D}{E}-based generative models},
  author={Lee, Sangyun and Kim, Beomsu and Ye, Jong Chul},
  booktitle={International Conference on Machine Learning},
  pages={18957--18973},
  year={2023},
  organization={PMLR}
}

@inproceedings{didiscaling,
  title={Scaling Atomistic Protein Binder Design with Generative Pretraining and Test-Time Compute},
  author={Didi, Kieran and Zhang, Zuobai and Zhou, Guoqing and Reidenbach, Danny and Cao, Zhonglin and Cha, Sooyoung and Geffner, Tomas and Dallago, Christian and Tang, Jian and Bronstein, Michael M and others},
  booktitle={The Fourteenth International Conference on Learning Representations}
}

@article{morehead2026zatom,
  title={Zatom-1: A Multimodal Flow Foundation Model for 3D Molecules and Materials},
  author={Morehead, Alex and Cretu, Miruna and Panescu, Antonia and Anand, Rishabh and Weiler, Maurice and Perez, Tynan and Blau, Samuel and Farrell, Steven and Bhimji, Wahid and Jain, Anubhav and Sahasrabuddhe, Hrushikesh and Li{\`o}, Pietro and Jaakkola, Tommi and G{\'o}mez-Bombarelli, Rafael and Ying, Rex and Erichson, N. Benjamin and Mahoney, Michael W.},
  journal={arXiv preprint arXiv:2602.22251},
  year={2026}
}

@article{limelucidating,
  title={Elucidating the Design Choice of Probability Paths in Flow Matching for Forecasting},
  author={Lim, Soon Hoe and Wang, Yijin and Yu, Annan and Hart, Emma and Mahoney, Michael W and Li, Sherry and Erichson, N Benjamin},
  journal={Transactions on Machine Learning Research},
    year={2024}
}

@article{lipman2024flow,
  title={Flow matching guide and code},
  author={Lipman, Yaron and Havasi, Marton and Holderrieth, Peter and Shaul, Neta and Le, Matt and Karrer, Brian and Chen, Ricky TQ and Lopez-Paz, David and Ben-Hamu, Heli and Gat, Itai},
  journal={arXiv preprint arXiv:2412.06264},
  year={2024}
}

@inproceedings{jinpyramidal,
  title={Pyramidal Flow Matching for Efficient Video Generative Modeling},
  author={Jin, Yang and Sun, Zhicheng and Li, Ningyuan and Xu, Kun and Jiang, Hao and Zhuang, Nan and Huang, Quzhe and Song, Yang and Mu, Yadong and Lin, Zhouchen},
  booktitle={The Thirteenth International Conference on Learning Representations}
}

@inproceedings{liu2023flow,
  title={Flow Straight and Fast: Learning to Generate and Transfer Data with Rectified Flow},
  author={Liu, Xingchao and Gong, Chengyue and Liu, Qiang},
  booktitle={The Eleventh International Conference on Learning Representations (ICLR)},
  year={2023}
}

@article{hu2024flowts,
  title={FlowTS: Time Series Generation via Rectified Flow},
  author={Hu, Yang and Wang, Xiao and Ding, Zezhen and Wu, Lirong and Zhang, Huatian and Li, Stan Z and Wang, Sheng and Zhang, Jiheng and Li, Ziyun and Chen, Tianlong},
  journal={arXiv preprint arXiv:2411.07506},
  year={2024}
}

@article{sriram2024flowllm,
  title={Flowllm: Flow matching for material generation with large language models as base distributions},
  author={Sriram, Anuroop and Miller, Benjamin K and Chen, Ricky T and Wood, Brandon M},
  journal={Advances in Neural Information Processing Systems},
  volume={37},
  pages={46025--46046},
  year={2024}
}

@inproceedings{yang2024lipschitz,
  title={Lipschitz Singularities in Diffusion Models},
  author={Yang, Zhantao and Feng, Ruili and Zhang, Han and Shen, Yujun and Zhu, Kai and Huang, Lianghua and Zhang, Yifei and Liu, Yu and Zhao, Deli and Zhou, Jingren and others},
  booktitle={ICLR},
  year={2024}
}

@inproceedings{esser2024scaling,
  title={Scaling rectified flow transformers for high-resolution image synthesis},
  author={Esser, Patrick and Kulal, Sumith and Blattmann, Andreas and Entezari, Rahim and M{\"u}ller, Jonas and Saini, Harry and Levi, Yam and Lorenz, Dominik and Sauer, Axel and Boesel, Frederic and others},
  booktitle={Forty-First International Conference on Machine Learning},
  year={2024}
}

@article{lu2022dpm,
  title={Dpm-solver: A fast ode solver for diffusion probabilistic model sampling in around 10 steps},
  author={Lu, Cheng and Zhou, Yuhao and Bao, Fan and Chen, Jianfei and Li, Chongxuan and Zhu, Jun},
  journal={Advances in Neural Information Processing Systems},
  volume={35},
  pages={5775--5787},
  year={2022}
}

@article{zhao2023unipc,
  title={Unipc: A unified predictor-corrector framework for fast sampling of diffusion models},
  author={Zhao, Wenliang and Bai, Lujia and Rao, Yongming and Zhou, Jie and Lu, Jiwen},
  journal={Advances in Neural Information Processing Systems},
  volume={36},
  pages={49842--49869},
  year={2023}
}

@inproceedings{song2021denoising,
title={Denoising Diffusion Implicit Models},
author={Jiaming Song and Chenlin Meng and Stefano Ermon},
booktitle={International Conference on Learning Representations},
year={2021},
}

@article{karras2022elucidating,
  title={Elucidating the design space of diffusion-based generative models},
  author={Karras, Tero and Aittala, Miika and Aila, Timo and Laine, Samuli},
  journal={Advances in Neural Information Processing Systems},
  volume={35},
  pages={26565--26577},
  year={2022}
}

@book{bach2024learning,
  title={Learning Theory from First Principles},
  author={Bach, Francis},
  year={2024},
  publisher={MIT Press}
}

@article{chen2024trajectory,
  title={On the trajectory regularity of {O}{D}{E}-based diffusion sampling},
  author={Chen, Defang and Zhou, Zhenyu and Wang, Can and Shen, Chunhua and Lyu, Siwei},
  journal={arXiv preprint arXiv:2405.11326},
  year={2024}
}

@article{hairer2006geometric,
  title={Geometric numerical integration},
  author={Hairer, Ernst and Hochbruck, Marlis and Iserles, Arieh and Lubich, Christian},
  journal={Oberwolfach Reports},
  volume={3},
  number={1},
  pages={805--882},
  year={2006}
}

@inproceedings{
salimans2022progressive,
title={Progressive Distillation for Fast Sampling of Diffusion Models},
author={Tim Salimans and Jonathan Ho},
booktitle={International Conference on Learning Representations},
year={2022},
url={https://openreview.net/forum?id=TIdIXIpzhoI}
}

@inproceedings{song2023consistency,
  title={Consistency models},
  author={Song, Yang and Dhariwal, Prafulla and Chen, Mark and Sutskever, Ilya},
  booktitle={Proceedings of the 40th International Conference on Machine Learning},
  pages={32211--32252},
  year={2023}
}

@article{luo2023latent,
  title={Latent consistency models: Synthesizing high-resolution images with few-step inference},
  author={Luo, Simian and Tan, Yiqin and Huang, Longbo and Li, Jian and Zhao, Hang},
  journal={arXiv preprint arXiv:2310.04378},
  year={2023}
}

@inproceedings{zhu2022numerical,
  title={On numerical integration in neural ordinary differential equations},
  author={Zhu, Aiqing and Jin, Pengzhan and Zhu, Beibei and Tang, Yifa},
  booktitle={International Conference on Machine Learning},
  pages={27527--27547},
  year={2022},
  organization={PMLR}
}

@article{lim2026flow,
  title={Is Flow Matching Just Trajectory Replay for Sequential Data?},
  author={Lim, Soon Hoe and Lin, Shizheng and Mahoney, Michael W and Erichson, N Benjamin},
  journal={arXiv preprint arXiv:2602.08318},
  year={2026}
}

@inproceedings{
lipman2023flow,
title={Flow Matching for Generative Modeling},
author={Yaron Lipman and Ricky T. Q. Chen and Heli Ben-Hamu and Maximilian Nickel and Matthew Le},
booktitle={The Eleventh International Conference on Learning Representations },
year={2023},
}

@article{watson2021learning,
  title={Learning to efficiently sample from diffusion probabilistic models},
  author={Watson, Daniel and Ho, Jonathan and Norouzi, Mohammad and Chan, William},
  journal={arXiv preprint arXiv:2106.03802},
  year={2021}
}

@inproceedings{watson2022fast,
  title={Learning fast samplers for diffusion models by differentiating through sample quality},
  author={Watson, Daniel and Chan, William and Ho, Jonathan and Norouzi, Mohammad},
  booktitle={International Conference on Learning Representations},
  year={2022}
}

@inproceedings{sabour2024align,
  title={Align your steps: optimizing sampling schedules in diffusion models},
  author={Sabour, Amirmojtaba and Fidler, Sanja and Kreis, Karsten},
  booktitle={Proceedings of the 41st International Conference on Machine Learning},
  pages={42947--42975},
  year={2024}
}

@book{hairer1993solving,
  title={Solving ordinary differential equations I: Nonstiff problems},
  author={Hairer, Ernst and Wanner, Gerhard and N{\o}rsett, Syvert P},
  year={1993},
  publisher={Springer}
}

@book{butcher2016numerical,
  title={Numerical Methods for Ordinary Differential Equations},
  author={Butcher, John Charles},
  year={2016},
  publisher={John Wiley \& Sons}
}

@book{griffiths2010numerical,
  title={Numerical Methods for Ordinary Differential Equations: Initial Value Problems},
    author={Griffiths, David F and Higham, Desmond J},
  volume={5},
    year={2010},
    publisher={Springer Science \& Business Media}
}

@inproceedings{xue2024accelerating,
  title={Accelerating diffusion sampling with optimized time steps},
  author={Xue, Shuchen and Liu, Zhaoqiang and Chen, Fei and Zhang, Shifeng and Hu, Tianyang and Xie, Enze and Li, Zhenguo},
  booktitle={Proceedings of the IEEE/CVF Conference on Computer Vision and Pattern Recognition},
  pages={8292--8301},
  year={2024}
}

@article{williams2024score,
  title={Score-optimal diffusion schedules},
  author={Williams, Christopher and Campbell, Andrew and Doucet, Arnaud and Syed, Saifuddin},
  journal={Advances in Neural Information Processing Systems},
  volume={37},
  pages={107960--107983},
  year={2024}
}

@inproceedings{
benton2024nearly,
title={Nearly d-Linear Convergence Bounds for Diffusion Models via Stochastic Localization},
author={Joe Benton and Valentin De Bortoli and Arnaud Doucet and George Deligiannidis},
booktitle={The Twelfth International Conference on Learning Representations},
year={2024},
}

@inproceedings{tsimpos2025optimal,
  title={Optimal Scheduling of Dynamic Transport},
  author={Tsimpos, Panos and Zhi, Ren and Zech, Jakob and Marzouk, Youssef},
  booktitle={The Thirty Eighth Annual Conference on Learning Theory},
  pages={5441--5505},
  year={2025},
}

@article{li2026kinetic,
  title={A Kinetic-Energy Perspective of Flow Matching},
  author={Li, Ziyun and Hu, Huancheng and Lim, Soon Hoe and Li, Xuyu and Gao, Fei and Diao, Enmao and Ding, Zezhen and Vazirgiannis, Michalis and Bostrom, Henrik},
  journal={arXiv preprint arXiv:2602.07928},
  year={2026}
}

@inproceedings{naeem2020reliable,
  title={Reliable Fidelity and Diversity Metrics for Generative Models},
  author={Naeem, Muhammad Ferjad and Oh, Seong Joon and Uh, Youngjung and Choi, Yunjey and Yoo, Jaejun},
  booktitle={Proceedings of the 37th International Conference on Machine Learning (ICML)},
  year={2020},
  pages={7176--7185},
}

@inproceedings{kynkaanniemi2019improved,
  title={Improved Precision and Recall Metric for Assessing Generative Models},
  author={Kynk{\"a}{\"a}nniemi, Tuomas and Karras, Tero and Laine, Samuli and Lehtinen, Jaakko and Aila, Timo},
  booktitle={Advances in Neural Information Processing Systems (NeurIPS)},
  year={2019},
}

@book{villani2008optimal,
  title={Optimal Transport: Old and New},
  author={Villani, C{\'e}dric},
  year={2008},
  publisher={Springer},
}

@inproceedings{cuturi2013sinkhorn,
  title={Sinkhorn Distances: Lightspeed Computation of Optimal Transport},
  author={Cuturi, Marco},
  booktitle={Advances in Neural Information Processing Systems (NeurIPS)},
  year={2013},
}

@misc{bionemo_ets_tutorial,
  title  = {{Entropic Flow Matching for Optimal Time Scheduling}},
  author = {{NVIDIA BioNeMo}},
  year   = {2024},
  note   = {BioNeMo Framework Documentation, MoCo tutorial},
  url    = {https://docs.nvidia.com/bionemo-framework/latest/main/examples/bionemo-moco/entropic_time_scheduler_tutorial_cfm/},
}

@misc{flux2024,
  title  = {{FLUX.1}},
  author = {{Black Forest Labs}},
  year   = {2024},
  url    = {https://github.com/black-forest-labs/flux},
}

@article{gupta2026quantifying,
  title={Quantifying Epistemic Uncertainty in Diffusion Models},
  author={Gupta, Aditi and Meyer, Raphael A and Yaniv, Yotam and Chen, Elynn and Erichson, N Benjamin},
  journal={arXiv preprint arXiv:2602.09170},
  year={2026}
}

@misc{diffusers_flux_pipeline,
  title  = {{FluxPipeline} Documentation},
  author = {{Hugging Face Diffusers}},
  year   = {2024},
  note   = {Default inference budget: 50 steps},
  url    = {https://huggingface.co/docs/diffusers/api/pipelines/flux},
}

\clearpage
\appendix

\section*{Appendix}

This appendix is organized as follows. In App. \ref{app:relatedwork} we discuss related work in detail and position SharpEuler relative to recent works. In App. \ref{app:detailed_algo} we provide the detailed algorithms for SharpEuler. In App. \ref{app:fulltheory} we provide full mathematical details and proofs for the three principles that we use to justisfy SharpEuler. In App. \ref{app:exp_details}-\ref{app:results} we provide details on the experiments and additional results.

\section{Related Work}
\label{app:relatedwork}


\textbf{Fast samplers and solver design.}
Much of the early work on fast generation was developed for diffusion and score-based models, where sampling also requires many sequential neural network evaluations. DDIM~\cite{song2021denoising} showed that deterministic sampling can greatly reduce the number of denoising steps. DPM-Solver~\cite{lu2022dpm} and UniPC~\cite{zhao2023unipc} design higher-order solvers that exploit the structure of diffusion ODEs, while EDM~\cite{karras2022elucidating} clarified
how solver choices, parameterization, and noise schedules interact. Other approaches use additional training to obtain very-few-step samplers, including Progressive Distillation~\cite{salimans2022progressive}, Consistency
Models~\cite{song2023consistency}, and Latent Consistency
Models~\cite{luo2023latent}. These works address the same practical bottleneck: sample quality per model evaluation. In contrast, we study pretrained flow matching models, keep the Euler solver fixed, and change only where a fixed number of evaluations are placed.

\textbf{Timestep schedules and time reparameterization.}
A complementary line of work treats the timestep schedule itself as an object of design. Prior work has optimized discrete DDPM schedules using dynamic programming and ELBO objectives~\cite{watson2021learning}, differentiated
through sample quality to search for fast samplers~\cite{watson2022fast},
optimized timesteps for a chosen numerical ODE solver~\cite{xue2024accelerating}, and shown that the best schedule can depend on the solver, model, and data~\cite{sabour2024align}. Other work derives path-traversal costs and discretizes the path to minimize them~\cite{williams2024score}. Entropic time schedulers (ETS) ~\cite{stancevic2025entropic} also view schedules as time
reparameterizations, but allocate steps using conditional-entropy information or rescaled conditional-entropy production. GITS
\cite{chen2024trajectory} uses dynamic programming over trajectory costs to select timesteps. These works show that uniform schedules are often not the right default. SharpEuler shares this motivation, but uses a different signal: it estimates trajectory sharpness of a pretrained flow matching model and converts this solver-aware profile into a fixed Euler grid.

\textbf{Flow matching, path geometry, and regularity.}
Flow matching~\cite{lipman2023flow}, rectified flow~\cite{liu2023flow}, and large-scale rectified flow models~\cite{esser2024scaling} have made
continuous-time vector field models a practical alternative to traditional diffusion formulations. In these models, sampling is again an ODE discretization problem. Some recent work studies timestep or noise sampling choices during training~\cite{esser2024scaling}, while other work aims to reduce trajectory curvature during training~\cite{lee2023minimizing}. Recent work has also
studied time-dependent sensitivity in diffusion networks
\cite{yang2024lipschitz}, discretization error in diffusion sampling
\cite{benton2024nearly}, time parameterizations that improve velocity regularity in dynamic transport~\cite{tsimpos2025optimal}, and trajectory energy as a diagnostic for flow matching samples~\cite{li2026kinetic}.  In particular, KTS in \cite{li2026kinetic}  suppresses terminal singularities and memorization by damping the learned velocity field near the data endpoint. In contrast, SharpEuler leaves the dynamics unchanged and adapts the time discretization to better resolve sharp changes in the velocity field. Since large velocity magnitude does not necessarily imply large velocity variation, the two methods address different issues: KTS reshapes the sampling dynamics, while SharpEuler improves numerical resolution.  These works are related in spirit because they treat path geometry as important. SharpEuler uses this perspective at inference time: rather than changing the training path or model, it measures finite-difference acceleration along sampling trajectories, the signal selected by Euler local truncation error and backward error analysis.

\textbf{Positioning of SharpEuler.}
SharpEuler is, to our knowledge, the first timestep scheduler for learned flow samplers that derives its schedule from an explicitly solver-aware Euler error proxy and estimates this profile using offline finite differences of the learned velocity field. While prior work has optimized timestep allocation using information-theoretic, spectral, or black-box objectives, SharpEuler targets the trajectory acceleration profile that controls Euler discretization error. This profile appears both in Euler local truncation error and in the leading backward error perturbation. The resulting square-root schedule is the closed-form minimizer of a continuous Euler-risk objective, and our statistical stability result shows that the empirically calibrated schedule
is near-oracle for the generated distribution.

\textbf{Limitations of SharpEuler.} SharpEuler assumes that the pretrained model can be sampled accurately on grids different from the one used during training or deployment. This assumption may fail for distilled models, or for models explicitly aligned to a prescribed sampling grid. For example, samplers used with Stable Diffusion variants are often tuned jointly with particular schedules, and changing the grid can violate the assumptions under which the model or sampler was calibrated. SharpEuler is also solver-specific. The schedule is derived from a leading local-error proxy for forward Euler, so the resulting grid is designed for a one-stage Euler update. Higher-order solvers, multistep solvers, or adaptive methods may require different error indicators and different timestep allocation rules. Extending sharpness-aware calibration to these solvers is an interesting direction for future work.

\section{Detailed Algorithms} \label{app:detailed_algo}
In this section, we provide detailed algorithms for SharpEuler. 
We use the reverse-time convention (using $s_k$ for the time indices), and consider a uniform reference grid  ($\Delta_i := 1/N$ for all $i$).

\begin{algorithm}[!h]
\caption{Offline sharpness profile calibration}
\label{alg:offline_calibration_2}
\begin{algorithmic}[1]
\REQUIRE Calibration inputs $\{c_j\}_{j=1}^M$, model $v_{\theta}$,
reverse-time grid $1=s_0>\cdots>s_N=0$, exponent $\gamma>0$,
smoothing bandwidth $\sigma>0$, small floor $\varepsilon_a\geq0$
\ENSURE Empirical inverse-CDF map $F_{\rm emp}^{-1}$

\FOR{$j=1$ \textbf{to} $M$}
    \STATE Sample $x_j\sim\mathcal N(0,I)$
    \FOR{$i=0$ \textbf{to} $N-1$}
        \STATE $V_{j,i}\leftarrow v_\theta(x_j,s_i;c_j)$
        \STATE $x_j\leftarrow x_j-(s_i-s_{i+1})V_{j,i}$
    \ENDFOR
\ENDFOR

\FOR{$i=0$ \textbf{to} $N-2$}
    \STATE $d_i\leftarrow
    \frac12\big(|s_i-s_{i+1}|+|s_{i+1}-s_{i+2}|\big)$
    \STATE $\widehat{\bar a}_i
    \leftarrow
    \frac1M\sum_{j=1}^M
    \left\|
        \frac{V_{j,i+1}-V_{j,i}}{d_i}
    \right\|_2$
\ENDFOR

\STATE Smooth the exponent-shaped signal:
\[
    \widetilde I_{\gamma}
    \leftarrow
    \mathrm{Smooth}_{\sigma}
    \left(
            (\widehat{\bar a}+\varepsilon_a)^\gamma
    \right).
\]

\STATE Normalize masses:
\[
    \widehat w_i
    \leftarrow
    \frac{\widetilde I_{\gamma,i}}
    {\sum_{r=0}^{N-2}\widetilde I_{\gamma,r}},
    \qquad i=0,\ldots,N-2.
\]

\STATE Assign forward-time support points:
\[
    \widetilde t_i
    \leftarrow
    1-\frac{s_i+s_{i+2}}{2},
    \qquad i=0,\ldots,N-2.
\]

\STATE Sort $(\widetilde t_i,\widehat w_i)$ in ascending $\widetilde t_i$,
build the empirical CDF, and add a left endpoint anchor at $t=0$.
\STATE \textbf{return} the linearly interpolated inverse-CDF map
$F_{\rm emp}^{-1}:[0,1]\to[0,1]$.
\end{algorithmic}
\end{algorithm}

\begin{algorithm}[!h]
\caption{Online SharpEuler sampling}
\label{alg:online_sampling_1}
\begin{algorithmic}[1]
\REQUIRE Empirical inverse-CDF map $F_{\rm emp}^{-1}$, budget $B$,
model $v_\theta$, conditioning input $c$
\ENSURE Sample $x^{(B)}$

\STATE Compute forward-time quantiles and convert to reverse time:
\[
    \widetilde t_b\leftarrow F_{\rm emp}^{-1}(b/B),
    \qquad
    s_b\leftarrow 1-\widetilde t_b,
    \qquad
    b=0,\ldots,B.
\]

\STATE Sample $x^{(0)}\sim\mathcal N(0,I)$
\FOR{$k=0$ \textbf{to} $B-1$}
    \STATE $x^{(k+1)}
    \leftarrow
    x^{(k)}
    -
    (s_k-s_{k+1})v_\theta(x^{(k)},s_k;c)$
\ENDFOR
\STATE \textbf{return} $x^{(B)}$
\end{algorithmic}
\end{algorithm}

Here \(\mathrm{Smooth}_{\sigma}\) denotes normalized Gaussian convolution in
index space with reflect padding:
\[
    \bigl(\mathrm{Smooth}_{\sigma}(x)\bigr)_i
    = \sum_{k=-r}^{r}
    K_k x_{i+k},
    \qquad
    K_k
    =
    \frac{\exp(-k^2/(2\sigma^2))}
    {\sum_{j=-r}^{r}\exp(-j^2/(2\sigma^2))},
    \qquad
    r=\max\{1,\lfloor 3\sigma\rfloor\},
\]
where \(x_{i+k}\) is interpreted using reflect padding whenever
\(i+k\notin\{0,\ldots,L-1\}\).


\section{Theoretical Justifications and Results} \label{app:fulltheory}
We assume throughout this section that the learned velocity field  is sufficiently smooth on the region visited by the sampled trajectories. In particular, all
derivatives and Taylor expansions below are assumed to be well defined and
uniformly bounded along these trajectories.

In this section, we provide full details on the three principles that we use to justify SharpEuler: (I) numerical, (II) variational, and (III) statistical. Recall that, as in Section \ref{sec:theory_sharpeuler}, we use the forward-time convention for the ODE sampler in our theoretical analysis.

\subsection{Principle I: Trajectory Acceleration Arises Naturally From Local Truncation Error and Backward Error Analysis}
Both local truncation error and backward error analyses are standard in numerical analysis of ODEs, but we include the details here for  completeness and to highlight why trajectory acceleration is the relevant sharpness signal. 

 Let \(x(t;x_0)\) denote the
solution of the forward-time ODE
\begin{equation}
    \dot x(t) = u_\theta(x(t),t), \qquad x(0)=x_0,
    \label{eq:learned_velocity_ode}
\end{equation}
with \(x_0\sim p_0\). We define the pathwise acceleration magnitude by
$a(t;x_0)
    =
    \|\ddot x(t;x_0)\|.$
Note that we suppress conditioning variables for readability.

The following result shows that this quantity is precisely the leading
quantity selected by the local truncation error of Euler discretization.

\begin{proposition}[Euler local truncation error]
\label{prop_detailedETA}
Let \(x:[t,t+h_0]\to\mathbb R^d\) be an exact solution of
\eqref{eq:learned_velocity_ode}, and suppose \(x\in C^3([t,t+h_0];\mathbb R^d)\).
If forward Euler is initialized at the current state \(x(t)\) and takes one step
of size \(0<h\leq h_0\) in forward time, then
\begin{equation}
    x(t+h)
    -
    \left[
        x(t)+h u_\theta(x(t),t)
    \right]
    =
    \frac{h^2}{2}\ddot x(t)
    +
    R_3(t,h),
    \label{eq:euler_lte_main}
\end{equation}
where
$\|R_3(t,h)\|
    \leq
    \frac{h^3}{6}
    \sup_{\xi\in[t,t+h]}
    \|x^{(3)}(\xi)\|.$
In particular,
\begin{equation}
    x(t+h)
    -
    \left[
        x(t)+h u_\theta(x(t),t)
    \right]
    =
    \frac{h^2}{2}\ddot x(t)
    +
    \mathcal O(h^3),
    \qquad \text{as } h\to0.
\end{equation}
Consequently, the leading-order one-step local truncation error is controlled
by \(h^2\|\ddot x(t)\|\).
\end{proposition}

\begin{proof}
By Taylor's theorem with remainder applied to the trajectory \(x(\cdot)\),
there exists a remainder \(R_3(t,h)\) satisfying
\[
    x(t+h)
    =
    x(t)
    +
    h\dot x(t)
    +
    \frac{h^2}{2}\ddot x(t)
    +
    R_3(t,h),
\]
with
\[
    \|R_3(t,h)\|
    \leq
    \frac{h^3}{6}
    \sup_{\xi\in[t,t+h]}
    \|x^{(3)}(\xi)\|.
\]
Since \(x(t)\) solves the forward-time ODE,
$\dot x(t)=u_\theta(x(t),t),$
subtracting the Euler update $x(t)+h u_\theta(x(t),t)$
gives Eq.~\eqref{eq:euler_lte_main}. Therefore,
\[
    \left\|
    x(t+h)
    -
    \left[
        x(t)+h u_\theta(x(t),t)
    \right]
    \right\|
    \leq
    \frac{h^2}{2}\|\ddot x(t)\|
    +
    \frac{h^3}{6}
    \sup_{\xi\in[t,t+h]}
    \|x^{(3)}(\xi)\|.
\]
Thus, as \(h\to0\), the leading term is
\(\frac{h^2}{2}\ddot x(t)\), so the leading local truncation error is governed
by \(h^2\|\ddot x(t)\|\), up to the constant factor \(1/2\).
\end{proof}

The same sharpness signal is also selected by backward error analysis
\cite{hairer2006geometric}. The central idea of backward error analysis is to
interpret a numerical method not merely as an approximation to the original
ODE, but as the exact flow, up to a prescribed order in the step size, of a
nearby step-size-dependent ODE. Thus, instead of asking only how far one Euler
step is from the exact flow of
\[
    \dot x = u_\theta(x,t),
\]
one asks which modified differential equation Euler is solving. In the present
setting, this means seeking a modified vector field
\[
    u_{\theta,h}(x,t)
    =
    u_\theta(x,t)
    +
    h u_{\theta,1}(x,t)
    +
    h^2 u_{\theta,2}(x,t)
    +\cdots
\]
whose exact flow over one step of length \(h\) matches the Euler update to a
prescribed order. The difference between \(u_{\theta,h}\) and \(u_\theta\)
then describes the finite-step bias induced by the solver. Since few-step
sampling uses non-negligible step sizes, the modified equation viewpoint gives
a useful way to identify where the numerical method is most strongly biased. This viewpoint has also been used to study numerical integration in Neural ODEs~\cite{zhu2022numerical}.

The following proposition makes this connection explicit.

\begin{proposition}[Informal: backward error interpretation]
\label{prop:backward_error_interpretation}
Consider the forward-time non-autonomous ODE
\begin{equation}
    \dot x = u_\theta(x,t),
    \label{eq:generic_velocity_ode}
\end{equation}
where \(u_\theta\) is sufficiently smooth (e.g., \(C^3\) locally in \((x,t)\)).
Forward Euler applied to \eqref{eq:generic_velocity_ode} has one-step map
\begin{equation}
    \psi_h(t,x)=x+h u_\theta(x,t).
\end{equation}
There exists a formal modified vector field, at the level of formal power
series in \(h\), of the form
\begin{equation}
    u_{\theta,h}(x,t)
    =
    u_\theta(x,t)
    +
    h u_{\theta,1}(x,t)
    +
    \mathcal O(h^2),
    \label{eq:modified_vector_field_expansion}
\end{equation}
whose exact time-\(h\) flow matches that of the Euler map up to local error
\(\mathcal O(h^3)\), as $h \to 0$. Its leading correction term is
\begin{equation}
    u_{\theta,1}(x,t)
    =
    -\frac12
    \left(
        \partial_t u_\theta(x,t)
        +
        \nabla_x u_\theta(x,t)u_\theta(x,t)
    \right).
    \label{eq:bea_utheta_1}
\end{equation}
Equivalently,
\begin{equation}
    u_{\theta,h}(x,t)
    =
    u_\theta(x,t)
    -
    \frac{h}{2}
    \left(
        \partial_t u_\theta(x,t)
        +
        \nabla_x u_\theta(x,t)u_\theta(x,t)
    \right)
    +
    \mathcal O(h^2).
    \label{eq:bea_leading_perturbation}
\end{equation}
Along exact trajectories of \eqref{eq:generic_velocity_ode}, the leading
modified equation perturbation is therefore proportional to
\(-\frac{h}{2}\ddot x(t)\).
\end{proposition}

\begin{proof}
Let the modified ODE be
\begin{equation}
    \dot{\widetilde x}
    =
    u_{\theta,h}(\widetilde x,t),
\end{equation}
with \(u_{\theta,h}\) of the form
\eqref{eq:modified_vector_field_expansion}. Starting from
\(\widetilde x(t)=x\), Taylor's theorem for the modified trajectory gives
\begin{align}
    \widetilde x(t+h)
    &=
    x
    +
    h u_{\theta,h}(x,t)
    +
    \frac{h^2}{2}
    \left[
        \partial_t u_{\theta,h}(x,t)
        +
        \nabla_x u_{\theta,h}(x,t)u_{\theta,h}(x,t)
    \right]
    +
    \mathcal O(h^3)
    \label{eq:modified_flow_taylor}
\end{align}
as $h \to 0$.
Substituting
\[
    u_{\theta,h}(x,t)
    =
    u_\theta(x,t)
    +
    h u_{\theta,1}(x,t)
    +
    \mathcal O(h^2)
\]
into \eqref{eq:modified_flow_taylor} and retaining terms through order \(h^2\)
gives
\begin{align}
    \widetilde x(t+h)
    &=
    x
    +
    h u_\theta(x,t)
    +
    h^2
    \left[
        u_{\theta,1}(x,t)
        +
        \frac12
        \left(
            \partial_t u_\theta(x,t)
            +
            \nabla_x u_\theta(x,t)u_\theta(x,t)
        \right)
    \right]
    +
    \mathcal O(h^3).
\end{align}
To match the forward Euler map
\[
    \psi_h(t,x)=x+h u_\theta(x,t)
\]
up to \(\mathcal O(h^3)\), the order-\(h^2\) coefficient must vanish. Hence,
\[
    u_{\theta,1}(x,t)
    =
    -\frac12
    \left(
        \partial_t u_\theta(x,t)
        +
        \nabla_x u_\theta(x,t)u_\theta(x,t)
    \right),
\]
which proves \eqref{eq:bea_utheta_1} and
\eqref{eq:bea_leading_perturbation}. Finally, along an exact trajectory, by the chain rule,
\[
    \partial_t u_\theta(x(t),t)
    +
    \nabla_x u_\theta(x(t),t)u_\theta(x(t),t)
    =
    \frac{d}{dt}u_\theta(x(t),t)
    =
    \ddot x(t).
\]
Therefore the leading backward error perturbation is proportional to
\(-h\ddot x(t)/2\).
\end{proof}

The preceding results show that the same quantity,
\begin{equation}
    \left\|
        \partial_t u_\theta(x(t),t)
        +
        \nabla_x u_\theta(x(t),t)u_\theta(x(t),t)
    \right\|
    =
    \|\ddot x(t)\|,
    \label{eq:sharpness_material_norm}
\end{equation}
is selected by two complementary numerical perspectives. From the local
truncation error viewpoint, it controls the leading one-step Euler error. From
the backward error viewpoint, it controls the leading perturbation in the
modified vector field implicitly integrated by Euler. Therefore, regions with
large trajectory acceleration are, to leading order, regions where Euler is both
locally less accurate and more strongly biased toward modified dynamics. This
provides a numerical justification for using velocity differences along
calibration trajectories as a sharpness signal in SharpEuler.

\subsection{Principle II: Power Scalings in SharpEuler Arise from Variational Allocation Problems}

The risk objective induced by the leading Euler local truncation error (Euler-risk) and the profile-matching objective in the main paper can be viewed as two members of a variational family. We introduce this family first, and then recover the two cases relevant for SharpEuler as special
cases.

Let $I:[0,1]\to(0,\infty)$ be a positive importance profile, and let \(\rho(t)\) be a normalized timestep
density.
For \(\beta\geq0\), define
\begin{equation}
    \psi_\beta(r)
    =
    \begin{cases}
    \dfrac{r^{-\beta}-1}{\beta}, & \beta>0,\\[1.2ex]
    -\log r, & \beta=0,
    \end{cases}
    \qquad r>0.
    \label{eq:psi_beta_definition}
\end{equation}
which is continuous at \(\beta=0\). 
Given \(I\), define
\begin{equation}
    \mathcal F_\beta(\rho;I)
    =
    \int_0^1
    I(t)\psi_\beta(\rho(t))
    \,dt,
    \qquad
    \rho(t)>0,
    \qquad
    \int_0^1\rho(t)\,dt=1.
    \label{eq:unified_variational_objective}
\end{equation}
For \(\beta>0\), this is, up to an additive constant independent of \(\rho\),
equivalent to
\begin{equation}
    \mathcal F_\beta(\rho;I)
    =
    \frac1\beta
    \int_0^1
    I(t)\rho(t)^{-\beta}
    \,dt
    +\mathrm{constant}.
\end{equation}
For \(\beta=0\), it becomes the weighted cross-entropy objective
\begin{equation}
    \mathcal F_0(\rho;I)
    =
    -
    \int_0^1
    I(t)\log \rho(t)
    \,dt.
\end{equation}

The following proposition characterizes the unique minimizer of $ \mathcal F_\beta(\rho;I)$ over the class of normalized densities $\rho > 0$.

\begin{proposition}[Variational timestep density]
\label{prop:unified_variational_density}
Let \(I:[0,1]\to(0,\infty)\) be continuous. For every \(\beta\geq0\), the
unique minimizer of
\begin{equation}
    \min_{\rho(t)>0,\ \int_0^1\rho(t)\,dt=1}
    \mathcal F_\beta(\rho;I)
\end{equation}
is
\begin{equation}
    \rho_\beta^\star(t)
    =
    \frac{
        I(t)^{1/(\beta+1)}
    }{
        \int_0^1 I(s)^{1/(\beta+1)}\,ds
    }.
    \label{eq:unified_density_solution}
\end{equation}
\end{proposition}

\begin{proof}
First suppose \(\beta>0\). Since the term
$-\frac1\beta\int_0^1I(t)\,dt$
does not depend on \(\rho\), minimizing \(\mathcal F_\beta\) is equivalent to
minimizing
\begin{equation}
    \frac1\beta
    \int_0^1
    I(t)\rho(t)^{-\beta}
    \,dt
\end{equation}
subject to \(\int_0^1\rho(t)\,dt=1\). The Lagrangian is
\begin{equation}
    \mathcal L(\rho,\lambda)
    =
    \frac1\beta
    \int_0^1
    I(t)\rho(t)^{-\beta}
    \,dt
    +
    \lambda
    \left(
        \int_0^1\rho(t)\,dt-1
    \right),
\end{equation}
where $\lambda$ is the Lagrange multiplier.
Taking the first variation with respect to \(\rho\) gives
\begin{equation}
    -I(t)\rho(t)^{-\beta-1}
    +
    \lambda
    =
    0.
\end{equation}
Therefore
\begin{equation}
    \rho(t)
    =
    \left(
        \frac{I(t)}{\lambda}
    \right)^{1/(\beta+1)}.
\end{equation}
Using the normalization constraint,
\begin{equation}
    1
    =
    \int_0^1\rho(t)\,dt
    =
    \lambda^{-1/(\beta+1)}
    \int_0^1 I(t)^{1/(\beta+1)}\,dt.
\end{equation}
Thus
\begin{equation}
    \lambda^{1/(\beta+1)}
    =
    \int_0^1 I(t)^{1/(\beta+1)}\,dt,
\end{equation}
and hence Eq.~\eqref{eq:unified_density_solution} follows.

For \(\beta=0\), the objective is
\begin{equation}
    \mathcal F_0(\rho;I)
    =
    -
    \int_0^1
    I(t)\log \rho(t)
    \,dt.
\end{equation}
The Lagrangian is
\begin{equation}
    \mathcal L(\rho,\lambda)
    =
    -
    \int_0^1
    I(t)\log \rho(t)
    \,dt
    +
    \lambda
    \left(
        \int_0^1\rho(t)\,dt-1
    \right).
\end{equation}
The first variation gives
\begin{equation}
    -\frac{I(t)}{\rho(t)}
    +
    \lambda
    =
    0.
\end{equation}
Thus $\rho(t)
    =
    \frac{I(t)}{\lambda}$, and
normalizing gives
\begin{equation}
    \rho_0^\star(t)
    =
    \frac{
        I(t)
    }{
        \int_0^1I(s)\,ds
    },
\end{equation}
which agrees with Eq.~\eqref{eq:unified_density_solution} at \(\beta=0\).

Finally, note that \(\psi_\beta\) is strictly convex on \(r>0\) for every
\(\beta\geq0\). Indeed, for \(\beta>0\),
\begin{equation}
    \psi_\beta''(r)
    =
    (\beta+1)r^{-\beta-2}>0,
\end{equation}
and for \(\beta=0\),
\begin{equation}
    \psi_0''(r)=\frac1{r^2}>0.
\end{equation}
Since \(I(t)>0\), the objective is strictly convex on the positive density
constraint set. Thus, the minimizer is unique.
\end{proof}

\paragraph{Special case 1: Euler-risk objective.}  Let $\bar a(t)
    =
    \mathbb E_{x_0\sim p_0}
    \|\ddot x(t;x_0)\|$
be the population mean acceleration sharpness. 
Given an Euler budget \(B\), the local step size induced by \(\rho\) scales as
\begin{equation}
    h(t)
    \approx
    \frac{1}{B\rho(t)}.
    \label{eq:local_step_density_scaling}
\end{equation}
Using Eq.~\eqref{eq:local_step_density_scaling}, the accumulated leading
Euler local-error proxy scales as
\begin{equation}
    \mathcal E[\rho]
    \approx
    \frac{1}{2B}
    \int_0^1
    \frac{
        \bar a(t)
    }{
        \rho(t)
    }
    \,dt.
    \label{eq:euler_error_functional}
\end{equation}
The factor \(1/\rho(t)\), rather than \(1/\rho(t)^2\), appears because the
local error is accumulated over the number of steps assigned to each time
region. A small interval \(dt\) receives approximately \(B\rho(t)\,dt\) Euler
steps, and each step contributes a leading local error of size
\(\frac12 h(t)^2\bar a(t)\). Therefore
\begin{equation}
    B\rho(t)\,dt
    \cdot
    \frac12
    \left(
        \frac{1}{B\rho(t)}
    \right)^2
    \bar a(t)
    =
    \frac{1}{2B}
    \frac{\bar a(t)}{\rho(t)}
    \,dt.
\end{equation}
The common factor \(1/(2B)\) does not affect the optimal density.
This Euler local-error objective is recovered from the unified family by taking
\begin{equation} \label{eq_contversion_sqrt}
    \beta=1,
    \qquad
    I(t)=\bar a(t).
\end{equation}
Then Eq.~\eqref{eq:unified_density_solution} gives
\begin{equation} 
    \rho_{\rm Euler}^\star(t)
    =
    \frac{
        \sqrt{\bar a(t)}
    }{
        \int_0^1\sqrt{\bar a(s)}\,ds
    }.
    \label{eq:euler_square_root_density}
\end{equation}
Therefore, the Euler local-error analysis selects the square-root density
\begin{equation}
    \rho(t)
    \propto
    \bar a(t)^{1/2}.
\end{equation}

\paragraph{Special case 2: cross-entropy and profile matching.}
The profile-matching or cross-entropy objective is recovered by taking $\beta=0$.
Given any positive importance profile \(I(t)>0\), the minimizer is
\begin{equation}
    \rho_I(t)
    =
    \frac{
        I(t)
    }{
        \int_0^1 I(s)\,ds
    }.
    \label{eq:profile_matching_density}
\end{equation}
Equivalently, \(\rho_I\) is the unique minimizer of
\begin{equation}
    \min_{\rho(t)>0,\ \int_0^1\rho(t)\,dt=1}
    -
    \int_0^1
    I(t)\log\rho(t)
    \,dt.
    \label{eq:profile_matching_objective}
\end{equation}
Now, for \(\beta=0\), define the normalized importance density
\begin{equation}
    \pi_I(t)
    =
    \frac{
        I(t)
    }{
        \int_0^1 I(s)\,ds
    }.
\end{equation}
Then
\begin{align}
    -
    \int_0^1
    I(t)\log \rho(t)
    \,dt
    &=
    \left(
        \int_0^1 I(s)\,ds
    \right)
    \left[
        -
        \int_0^1
        \pi_I(t)\log \rho(t)
        \,dt
    \right] \\
    &=
    \left(
        \int_0^1 I(s)\,ds
    \right)
    \left[
        H(\pi_I)
        +
        {\rm KL}(\pi_I\|\rho)
    \right].
\end{align}
Since the entropy \(H(\pi_I)\) does not depend on \(\rho\), minimizing the cross-entropy
objective is equivalent to minimizing the Kullback-Leibler divergence 
    $ {\rm KL}(\pi_I\|\rho)$.
The unique minimizer is therefore \(\rho=\pi_I\).

\paragraph{SharpEuler exponent schedules.}
SharpEuler with exponent \(\gamma>0\) can be viewed as an empirical estimator
of the solver-aware importance profile
\[
    I_\gamma(t)
    =
    \left(
        \mathbb E_{x_0\sim p_0}
        \|\ddot x(t;x_0)\|
    \right)^\gamma,
\]
with associated profile density
\[
    \pi_{I_\gamma}(t)
    =
    \frac{
        I_\gamma(t)
    }{
        \int_0^1 I_\gamma(s)\,ds
    }.
\]
As discussed earlier, the use of trajectory acceleration is motivated by the Euler local truncation
error and by the backward error analysis, both of which identify the material
derivative of the learned velocity field as the relevant  sharpness
signal.

\paragraph{Connection to entropy-based schedulers \cite{stancevic2025entropic}.}
Entropy-based schedules can be viewed as using an information-theoretic
importance profile. Following the notation of
\cite{stancevic2025entropic}, let \(H[x_0\mid x_t]\) denote the conditional
differential entropy of the data endpoint \(x_0\) given the intermediate state
\(x_t\),
\[
    H[x_0\mid x_t]
    =
    -\mathbb E_{p(x_0,x_t)}
    \left[
        \log p(x_0\mid x_t)
    \right].
\]
In entropic time, the time change is
\[
    \phi_{\rm ent}(t)=H[x_0\mid x_t],
\]
and in rescaled entropic time,
\[
    \phi_{\rm rent}(t)
    =
    \int_0^t
    \sigma(\tau)\dot H[x_0\mid x_\tau]\,d\tau.
\]
After orienting time so that \(\dot\phi(t)>0\), uniform spacing in the new time
variable induces the timestep density
\[
    \rho_\phi(t)
    =
    \frac{\dot\phi(t)}
    {\int_0^1\dot\phi(s)\,ds}.
\]
Thus entropy-based schedules use conditional-entropy production as their
importance profile, whereas SharpEuler uses solver sharpness. The square-root
SharpEuler density is special because it minimizes the Euler-risk
objective and satisfies \(\rho(t)\propto \bar a(t)^{1/2}\). An entropic schedule
recovers this density only when its induced profile \(\dot\phi(t)\) is
proportional to \(\bar a(t)^{1/2}\).

\paragraph{Specialization to discrete intervals.}
Let
\begin{equation}
    0=\tau_0<\tau_1<\cdots<\tau_N=1,
    \qquad
    \Delta_i=\tau_{i+1}-\tau_i,
    \qquad
    i=0,\ldots,N-1.
\end{equation}
A density \(\rho\) induces interval masses
\begin{equation}
    w_i
    =
    \int_{\tau_i}^{\tau_{i+1}}
    \rho(t)\,dt,
    \qquad
    \sum_{i=0}^{N-1} w_i=1.
\end{equation}
If the profile is approximately constant on each interval, with value
\(\phi_i>0\), then the profile-matching density
\begin{equation}
    \rho_\gamma(t)\propto \phi_i^\gamma,
    \qquad
    t\in[\tau_i,\tau_{i+1}),
\end{equation}
induces weights
\begin{equation}
    w_i^{(\gamma)}
    =
    \frac{
        \Delta_i\phi_i^\gamma
    }{
        \sum_{\ell=0}^{N-1}
        \Delta_\ell\phi_\ell^\gamma
    }.
    \label{eq:discrete_gamma_weights}
\end{equation}
In particular, for \(\gamma=1/2\),
\begin{equation}
    w_i^{(1/2)}
    =
    \frac{
        \Delta_i\sqrt{\phi_i}
    }{
        \sum_{\ell=0}^{N-1}
        \Delta_\ell\sqrt{\phi_\ell}
    }.
\end{equation}

\begin{proposition}[Discrete profile matching]
\label{prop:discrete_profile_matching}
Fix \(\gamma>0\) and \(\phi_i>0\). The weights in
Eq.~\eqref{eq:discrete_gamma_weights} are the unique minimizer of
\begin{equation}
    \min_{w_i>0,\ \sum_iw_i=1}
    -
    \sum_{i=0}^{N-1}
    \Delta_i\phi_i^\gamma
    \log w_i.
\end{equation}
\end{proposition}

\begin{proof}
Let \(q_i=\Delta_i\phi_i^\gamma\). The Lagrangian is
\begin{equation}
    \mathcal L(w,\lambda)
    =
    -
    \sum_iq_i\log w_i
    +
    \lambda
    \left(
        \sum_iw_i-1
    \right),
\end{equation}
where $\lambda$ is the Lagrange multiplier. 
The first-order optimality condition gives
\begin{equation}
    -\frac{q_i}{w_i}+\lambda=0,
\end{equation}
so \(w_i=q_i/\lambda\). The normalization constraint gives
\begin{equation}
    \lambda=\sum_iq_i,
\end{equation}
and hence
\begin{equation}
    w_i
    =
    \frac{
        \Delta_i\phi_i^\gamma
    }{
        \sum_{\ell=0}^{N-1}\Delta_\ell\phi_\ell^\gamma
    }.
\end{equation}
Strict convexity of \(-\log w_i\) gives uniqueness and concludes the proof.
\end{proof}

\subsection{Principle III: Statistical Stability Guarantee For the Calibrated Schedule}
\label{sec:stat_guarantee}

In the main paper, we introduce a practical offline calibration procedure for learned flow samplers that estimates where the learned vector field is numerically sharp and converts this profile into a fixed-budget timestep schedule (SharpEuler). In this subsection, we provide a statistical stability guarantee to show that the calibrated schedule is near-oracle at the level of the generated distribution at the terminal time $t=1$.
We focus on the square-root schedule
$\gamma=1/2$, which is the exponent selected by the leading Euler local-error
proxy. General exponent schedules are discussed in Remark~\ref{rmk_general_gamma}.

\paragraph{Motivation.}
The scheduler is calibrated offline from finitely many ODE trajectories. Hence the
schedule used at sampling time is random, since it depends on the calibration
set. Our goal here is to show that this calibration randomness is
controlled at the level of the final generated distribution. 

\paragraph{Setting.}
We use the forward-time  convention. Consider the learned flow ODE
\begin{equation}
    \dot{x}(t)
    =
    u_\theta(x(t),t),
    \qquad
    t\in[0,1],
    \label{eq:learned_flow_ode}
\end{equation}
where  $x_0\sim p_0$ is the initial noise  sampled from the base distribution $p_0$. Let $T(x_0):=x(1;x_0)$ denote the learned flow map, and define the learned flow distribution $\mu_1:=T_{\#}p_0$.

Let us fix an $N$ from now on, and let $0=\tau_0<\tau_1<\cdots<\tau_N=1$
be a reference grid, writing $\Delta_i := \tau_{i+1}-\tau_i$, $i=0,\ldots,N-1$. 
A schedule is represented by the weights
\begin{equation}
    w=(w_0,\ldots,w_{N-1}),
    \qquad
    w_i>0,
    \qquad
    \sum_{i=0}^{N-1}w_i=1.
\end{equation}
The interpretation is that interval $i$ receives fraction $w_i$ of a total Euler
budget $B$. Thus interval $i$ receives approximately $Bw_i$ steps, with local
step size approximately
$h_i
    \approx
    \frac{\Delta_i}{Bw_i}.$
Let  $T_w(x_0)$ be the Euler sampler output at $t=1$ using schedule $w$, and let
$\mu_w := (T_w)_{\#}p_0$
be the corresponding sampler distribution at $t=1$.

\paragraph{Local-error risk and implemented sharpness.}
For each trajectory initialized at $x_0$, let
$A_i(x_0)\geq 0$
be an interval-wise local-error coefficient on $[\tau_i,\tau_{i+1}]$.

A natural continuous coefficient is the interval average
\begin{equation}
    A_i(x_0)
    =
    \frac{1}{\Delta_i}
    \int_{\tau_i}^{\tau_{i+1}}
    \|\ddot{x}(t;x_0)\|
    \,dt.
    \label{eq:average_local_error_coefficient}
\end{equation}
Alternatively, one may instead take a dominating coefficient
such as
\begin{equation}
    A_i(x_0)
    =
    \sup_{t\in[\tau_i,\tau_{i+1}]}
    \|\ddot{x}(t;x_0)\|.
    \label{eq:sup_local_error_coefficient}
\end{equation}


The motivation comes from Euler local truncation error as discussed in Section \ref{sec:theory_sharpeuler}. The leading local Euler error is controlled by $h^2\|\ddot{x}(t)\|.$
If interval $i$ receives $Bw_i$ steps of size
$h_i
    \approx
    \frac{\Delta_i}{Bw_i},$
then the accumulated local-error proxy on interval $i$ is
\begin{equation}
    Bw_i
    A_i(x_0)
    \left(
        \frac{\Delta_i}{Bw_i}
    \right)^2
    =
    \frac{1}{B}
    \frac{
        A_i(x_0)\Delta_i^2
    }{
        w_i
    }.
    \label{eq:interval_error_proxy}
\end{equation}
Ignoring the common factor $1/B$, define the population local-error risk
\begin{equation}
    \mathcal J(w)
    :=
    \mathbb E_{x_0\sim p_0}
    \left[
        \sum_{i=0}^{N-1}
        \frac{
            A_i(x_0)\Delta_i^2
        }{
            w_i
        }
    \right].
    \label{eq:raw_population_risk}
\end{equation}
Equivalently,
\begin{equation}
    \mathcal J(w)
    =
    \sum_{i=0}^{N-1}
    \frac{
        \mathcal A_i\Delta_i^2
    }{
        w_i
    },
    \qquad
    \mathcal A_i
    :=
    \mathbb E_{x_0\sim p_0}
    [A_i(x_0)].
    \label{eq:raw_population_risk_mean_profile}
\end{equation}
The factor $1/w_i$ is essential: allocating more steps to interval $i$ decreases its local-error contribution.

\paragraph{Smoothed empirical calibration.}
Suppose we have $M$ independent calibration trajectories. For each
$m=1,\ldots,M$, draw $x_0^{(m)}\sim p_0$
independently, and define
$A_i^{(m)} := A_i(x_0^{(m)}).$
Note that in implementation, $A_i^{(m)}$ is replaced by the finite-difference proxy. 

The empirical raw profile is
\begin{equation}
    \widehat{\mathcal A}_i
    =
    \frac1M
    \sum_{m=1}^{M}
    A_i^{(m)}.
    \label{eq:empirical_raw_profile}
\end{equation}
The population raw profile is
\begin{equation}
    \mathcal A_i
    =
    \mathbb E_{x_0\sim p_0}[A_i(x_0)].
\end{equation}

For the square-root SharpEuler schedule, the exponent is applied after
averaging over calibration trajectories:
\begin{equation}
    \widehat g_i
    =
    \sqrt{
        \widehat{\mathcal A}_i+\varepsilon_a
    },
    \qquad
    \varepsilon_a > 0.
    \label{eq:empirical_sqrt_profile}
\end{equation}
Let \(K_\sigma(i,\ell)\geq0\) be Gaussian smoothing weights satisfying
\[
    \sum_{\ell=0}^{N-1}K_\sigma(i,\ell)=1.
\]
The profile is then smoothed:
\begin{equation}
    \widehat b_i
    =
    \sum_{\ell=0}^{N-1}
    K_\sigma(i,\ell)\widehat g_\ell.
    \label{eq:empirical_sqrt_then_smooth_profile}
\end{equation}
Since \(\widehat g_i\geq0\) and the smoothing weights are nonnegative,
\(\widehat b_i\geq0\). 

The calibrated weights are
\begin{equation}
    \widehat w_i
    =
    \frac{
        \Delta_i\widehat b_i
    }{
        \sum_{\ell=0}^{N-1}
        \Delta_\ell\widehat b_\ell
    }.
    \label{eq:empirical_calibrated_weights}
\end{equation}

The corresponding population square-root profile is
$g_i^\star
    =
    \sqrt{
        \mathcal A_i+\varepsilon_a
    },$
and its smoothed version is
\begin{equation}
    b_i^\star
    =
    \sum_{\ell=0}^{N-1}
    K_\sigma(i,\ell)g_\ell^\star.
    \label{eq:population_sqrt_then_smooth_profile}
\end{equation}
Let us also define the effective empirical and population coefficients
\begin{equation}
    \widehat\phi_i
    :=
    \widehat b_i^2,
    \qquad
    \phi_i^\star
    :=
    (b_i^\star)^2.
    \label{eq:effective_coefficients_sqrt_smooth}
\end{equation}
Then \(\widehat w\) is the minimizer of the empirical effective Euler-risk
proxy
\begin{equation}
    \widehat{\mathcal J}_{\sigma,\varepsilon_a}(w)
    =
    \sum_{i=0}^{N-1}
    \frac{
        \widehat\phi_i\Delta_i^2
    }{
        w_i
    }.
    \label{eq:empirical_effective_risk}
\end{equation}

Assuming that these weights are realized through the piecewise-constant scheduling density
\begin{equation}
    \widehat p(t)
    =
    \frac{\widehat w_i}{\Delta_i},
    \qquad
    t\in[\tau_i,\tau_{i+1}),
\end{equation}
we now analyze this idealized interval-allocation schedule.  The practical implementation
constructs the schedule by forming an empirical CDF from the calibrated masses
and evaluating its inverse by linear interpolation; this realizes the same
allocation up to integer-rounding, interpolation, and finite-grid boundary
effects.

\subsubsection{The Main Theorem}

Consider the schedule class
\begin{equation}
    \mathcal W_\omega
    :=
    \left\{
        w\in\mathbb R^N:
        w_i\geq \omega,\ 
        \sum_{i=0}^{N-1}w_i=1
    \right\},
    \qquad
    0<\omega<1/N.
    \label{eq:schedule_class}
\end{equation}
The lower bound prevents degenerate schedules that allocate vanishing mass to an
interval and makes the oracle comparison stable.

Let
\begin{equation}
    w^\star
    \in
    \arg\min_{w\in\mathcal W_\omega}
    \mathcal J(w)
    \label{eq:raw_oracle_schedule}
\end{equation}
be the best raw population-risk schedule in this class. Define the population effective objective
\begin{equation}
    \mathcal J_{\sigma,\varepsilon_a}(w)
    :=
    \sum_{i=0}^{N-1}
    \frac{
        \phi_i^\star\Delta_i^2
    }{
        w_i
    }.
\label{eq:population_smoothed_floored_objective}
\end{equation}
and the  bias due to smoothing and other procedures
\begin{equation}
    \beta_{\sigma,\varepsilon_a}
    :=
    \max_{0\leq i\leq N-1}
    \left|
        \phi_i^\star-\mathcal A_i
    \right|
    =
    \max_i
    \left|
        (b_i^\star)^2-\mathcal A_i
    \right|.
    \label{eq:smoothing_flooring_bias}
\end{equation}
This term measures the mismatch between the raw profile appearing in the Euler
error proxy and the  profile optimized by the algorithm. 

We make the following assumptions. 

\paragraph{(A1) Independent calibration trajectories.}
The calibration samples $x_0^{(1)},\ldots,x_0^{(M)}$
are i.i.d. from $p_0$.

\paragraph{(A2) Bernstein conditions.}
For every $i=0,\ldots,N-1$, there exist positive constants $R$ and $\nu$ such that
\begin{equation}
    \left|
        A_i(x_0)-\mathcal A_i
    \right|
    \leq R
\end{equation}
almost surely, and
\begin{equation}
    \operatorname{Var}_{x_0\sim p_0}
    \left(
        A_i(x_0)
    \right)
    \leq \nu^2.
\end{equation}
The boundedness assumption is a sufficient condition used to obtain a simple
Bernstein concentration bound. In practice, it can be enforced by clipping the calibrated sharpness values, or replaced by an appropriate tail assumption with the corresponding concentration bound.

\paragraph{(A3) Lower-bound and bounded effective-profile conditions.}
The closed-form empirical schedule $\widehat w$ lies in $\mathcal W_\omega$.
Likewise, the unconstrained minimizer of the population effective objective
$\mathcal J_{\sigma,\varepsilon_a}$ lies in $\mathcal W_\omega$.
In addition, assume on the calibration event under consideration that
\begin{equation}
    0\leq \widehat b_i\leq G_{\max},
    \qquad
    0\leq b_i^\star\leq G_{\max},
    \qquad
    i=0,\ldots,N-1.
    \label{eq:effective_profile_boundedness}
\end{equation}

A sufficient deterministic condition is that \(0\leq A_i(x_0)\leq A_{\max}\)
almost surely, in which case one may take
$G_{\max}=\sqrt{A_{\max}+\varepsilon_a}.$
A sufficient condition for the lower-bound part is that
$\Delta_i\geq \Delta_{\min}>0$
and
$\sqrt{\varepsilon_a}\leq \widehat b_i\leq G_{\max},$
$\sqrt{\varepsilon_a}\leq b_i^\star\leq G_{\max}.$
Then
\begin{equation}
    \widehat w_i
    =
    \frac{
        \Delta_i\widehat b_i
    }{
        \sum_{\ell}
        \Delta_\ell\widehat b_\ell
    }
    \geq
    \frac{
        \Delta_{\min}\sqrt{\varepsilon_a}
    }{
        G_{\max}
    },
\end{equation}
and similarly for the population effective weights. Hence one may choose
\begin{equation}
    \omega
    \leq
    \frac{
        \Delta_{\min}\sqrt{\varepsilon_a}
    }{
        G_{\max}
    }.
\end{equation}

\paragraph{(A4) Euler stability.}
For every $w\in\mathcal W_\omega$, the Euler sampler satisfies
\begin{equation}
    \|T_w(x_0)-T(x_0)\|
    \leq
    \frac{C_{\rm Eul}}{B}
    \sum_{i=0}^{N-1}
    \frac{
        A_i(x_0)\Delta_i^2
    }{
        w_i
    },
    \label{eq:euler_stability_assumption}
\end{equation}
where $C_{\rm Eul} > 0$ is a stability constant. 

Assumption (A4) is a standard nonuniform Euler stability bound. For example,
if $u_\theta$ is uniformly $L$-Lipschitz in $x$ and we take
$A_i(x_0) = \sup_{t\in[\tau_i,\tau_{i+1}]}  \|\ddot x(t;x_0)\|,$
then a discrete Gr\"onwall argument gives (A4) with $C_{\rm Eul}=e^L/2$, up to
the idealized interval-allocation approximation. We state (A4) as an assumption because learned
flow fields may have endpoint singularities \cite{li2026kinetic} or problem-dependent stability constants.

Define $D_2:=\sum_{i=0}^{N-1}\Delta_i^2.$
For \(\delta\in(0,1)\), define
\begin{equation}
    \eta_M(\delta)
    :=
    \sqrt{
        \frac{2\nu^2\log(2N/\delta)}{M}
    }
    +
    \frac{
        2R\log(2N/\delta)
    }{
        3M
    },
    \label{eq:eta_M_delta}
\end{equation}
and
\begin{equation}
    \alpha_M(\delta)
    :=
    \frac{
        \eta_M(\delta)
    }{
        2\sqrt{\varepsilon_a}
    },
    \qquad
    \zeta_M(\delta)
    :=
    2G_{\max}\alpha_M(\delta)
    \label{eq:zeta_M_delta}
\end{equation}

With these assumptions and discussions in place, we have the following statistical guarantee in terms of the 1-Wasserstein distance $W_1$.

\begin{theorem}[Statistical stability guarantee]
\label{thm:stat_guarantee}
Let $\delta \in (0,1)$, $N$, and $M$ be given. Under Assumptions (A1)--(A4), with probability at least $1-\delta$ over the
calibration set,
\begin{equation}
    W_1(\mu_{\widehat w},\mu_1)
    \leq
    \frac{C_{\rm Eul}}{B}
    \left[
        \mathcal J(w^\star)
        +
        \frac{2D_2}{\omega}
        \left(
            \zeta_M(\delta)+\beta_{\sigma,\varepsilon_a}
        \right)
    \right].
    \label{eq:main_oracle_bound}
\end{equation}
\end{theorem}

The high-probability statement above is over the random calibration trajectories.
Conditional on the calibrated schedule, $\mu_{\widehat w}$ is the law induced by
fresh samples $x_0\sim p_0$. Equivalently, treating \(G_{\max}\) and
\(\varepsilon_a\) as fixed constants,
\begin{equation}
    W_1(\mu_{\widehat w},\mu_1)
    \leq
    \frac{C_{\rm Eul}}{B}\mathcal J(w^\star)
    +
    O_{\mathbb P}
    \left(
        \frac{C_{\rm Eul}D_2}{B\omega}
        \sqrt{
            \frac{\log N}{M}
        }
    \right)
    +
    \frac{C_{\rm Eul}D_2}{B\omega}
    \beta_{\sigma,\varepsilon_a}.
    \label{eq:main_oracle_bound_op}
\end{equation}

\begin{remark}
The oracle term $\frac{C_{\rm Eul}}{B}\mathcal J(w^\star)$ is the best Euler local-error proxy achievable by the schedule class.
When all $\mathcal A_i>0$ and the lower-bound constraint $w_i\geq \omega$ is
inactive, the oracle problem is
\begin{equation}
    \min_{w_i>0,\ \sum_iw_i=1}
    \sum_{i=0}^{N-1}
    \frac{
        \mathcal A_i\Delta_i^2
    }{
        w_i
    }.
\end{equation}
Its minimizer is
\begin{equation}
    w_i^\star
    =
    \frac{
        \Delta_i\sqrt{\mathcal A_i}
    }{
        \sum_{\ell=0}^{N-1}
        \Delta_\ell\sqrt{\mathcal A_\ell}
    },
    \label{eq:unconstrained_oracle_weights}
\end{equation}
and the optimal value is
\begin{equation}
    \mathcal J(w^\star)
    =
    \left(
        \sum_{i=0}^{N-1}
        \Delta_i\sqrt{\mathcal A_i}
    \right)^2.
    \label{eq:oracle_value_explicit}
\end{equation}
Thus, in the oracle case,
\begin{equation}
    W_1(\mu_{w^\star},\mu_1)
    \leq
    \frac{C_{\rm Eul}}{B}
    \left(
        \sum_{i=0}^{N-1}
        \Delta_i\sqrt{\mathcal A_i}
    \right)^2.
\end{equation}
Since $\mathcal J(w^\star)$ does not depend on $B$, the oracle discretization
term scales as $O(1/B)$, matching the standard global convergence rate of Euler
integration. 

The sharpness-aware schedule improves the constant in this rate by
allocating more steps to intervals with larger local-error coefficients.  
The improvement is in the leading error constant: for the unconstrained oracle, applying Cauchy--Schwarz inequality:
\[
\mathcal J(w^\star)=\left(\sum_i\Delta_i\sqrt{\mathcal A_i}\right)^2
\leq
\sum_i\Delta_i\mathcal A_i
=
\mathcal J(w^{\rm unif}),
\]
where $w_i^{\rm unif}=\Delta_i$, corresponding to uniform step density in time. Thus the sharpness-aware oracle is never worse than uniform allocation for the
local-error proxy, with larger gains when the sharpness profile
is more nonuniform, while preserving the same \(O(1/B)\) Euler rate.

As $M\to\infty$, the finite-calibration term vanishes at rate
$O(M^{-1/2})$ up to a logarithmic factor in $N$, so the calibrated schedule
approaches the population schedule, up to the remaining gap due to
smoothing and other bias $\beta_{\sigma,\varepsilon_a}$. \\
\end{remark}

\begin{remark}
The theorem compares two distributions at $t=1$, i.e.
$\mu_{\widehat w}
    =
    (T_{\widehat w})_{\#}p_0$,
the law generated by the Euler sampler using the empirically calibrated
schedule, and $\mu_1=T_{\#}p_0$,
the law generated by the learned ODE flow.
It does not directly compare $\mu_{\widehat w}$ to the true data distribution.
A full distribution bound would need the model error term:
\begin{equation}
    W_1(\mu_{\widehat w},p_{\rm data})
    \leq
    W_1(\mu_{\widehat w},\mu_1)
    +
    W_1(\mu_1,p_{\rm data}).
\end{equation}
The scheduler controls only the first term. At a high level, the final sampler error can be decomposed as:
\begin{equation*}
    \text{final sampler error}
    \leq
    \text{best schedule error}
    +
    \text{finite calibration error}
    +
    \text{smoothing and other bias}.
\end{equation*}
Here, we compare the flow induced by the calibrated schedule to the  learned ODE flow. Thus the result controls numerical
sampling error, rather than model error due to neural network approximation and optimization. \\
\end{remark}

\begin{remark}[General exponent schedules] \label{rmk_general_gamma}
The implementation also allows schedules of the form
\[
    \widehat g_i^{(\gamma)}
    =
    \left(
        \widehat{\mathcal A}_i+\varepsilon_a
    \right)^\gamma,
    \qquad
    \widehat b_i^{(\gamma)}
    =
    \sum_\ell
    K_\sigma(i,\ell)
    \widehat g_\ell^{(\gamma)},
\]
followed by
\[
    w_i^{(\gamma)}
    =
    \frac{
        \Delta_i \widehat b_i^{(\gamma)}
    }{
        \sum_{\ell=0}^{N-1}
        \Delta_\ell \widehat b_\ell^{(\gamma)}
    }.
\]
The theorem above corresponds to $\gamma=1/2$, which is the exponent selected by the Euler local-error proxy. 
The statistical concentration part of the proof extends straightforwardly to these
schedules, with constants depending on $\gamma$. However, only the square-root case $\gamma=1/2$ is the minimizer
of the Euler local-error proxy
$\sum_i \frac{\mathcal A_i\Delta_i^2}{w_i}$.
Therefore the near-oracle statement in
Theorem~\ref{thm:stat_guarantee} is stated for the square-root SharpEuler
schedule. For $\gamma\neq1/2$, the analogous guarantee should be interpreted as stability around the corresponding population exponent schedule, rather than
as optimality for the Euler-error oracle.
\end{remark}

\subsubsection{Proofs}
We provide a detailed proof to Theorem \ref{thm:stat_guarantee}.  

First, we recall the continuous version of the minimization problem (see also  \eqref{eq_contversion_sqrt}-\eqref{eq:euler_square_root_density}).

\begin{proposition}
\label{prop:sharpness_density}
Let \(a:[0,1]\to(0,\infty)\) be continuous. Among all normalized positive
densities $\rho(t)>0$,  $\int_0^1 \rho(t)\,dt=1$,   the unique minimizer of
$\int_0^1
    \frac{a(t)}{\rho(t)}
    \,dt$
is
\begin{equation}
    \rho^\star(t)
    =
    \frac{\sqrt{a(t)}}
    {\int_0^1 \sqrt{a(s)}\,ds}.
    \label{eq:optimal_density}
\end{equation}
Moreover, the minimum value is
$\left(
        \int_0^1
        \sqrt{a(t)}
        \,dt
    \right)^2.$
\end{proposition}

\begin{proof}
We solve the constrained optimization problem
\begin{equation}
    \min_{\rho}
    J[\rho],
    \qquad
    J[\rho]
    =
    \int_0^1
    \frac{a(t)}{\rho(t)}
    \,dt,
    \label{eq:app_J_def}
\end{equation}
subject to
$\rho(t)>0,$ $\int_0^1 \rho(t)\,dt = 1.$
Introduce a Lagrange multiplier \(\lambda\in\mathbb{R}\) for the normalization
constraint and define
\begin{equation}
    \mathcal{L}[\rho,\lambda]
    =
    \int_0^1
    \frac{a(t)}{\rho(t)}
    \,dt
    +
    \lambda
    \left(
        \int_0^1 \rho(t)\,dt - 1
    \right).
    \label{eq:app_lagrangian}
\end{equation}
To compute the first variation, let \(\eta(t)\) be a smooth perturbation and
consider
\begin{equation}
    \rho_\epsilon(t)
    =
    \rho(t)+\epsilon\eta(t),
\end{equation}
for sufficiently small \(\epsilon\) so that \(\rho_\epsilon(t)>0\). Then
\begin{align}
    \mathcal{L}[\rho_\epsilon,\lambda]
    &=
    \int_0^1
    \frac{a(t)}{\rho(t)+\epsilon\eta(t)}
    \,dt
    +
    \lambda
    \left(
        \int_0^1
        \left[
            \rho(t)+\epsilon\eta(t)
        \right]
        dt
        -
        1
    \right).
\end{align}
Differentiating with respect to \(\epsilon\) at \(\epsilon=0\) gives
\begin{align}
    \left.
    \frac{d}{d\epsilon}
    \mathcal{L}[\rho_\epsilon,\lambda]
    \right|_{\epsilon=0}
    &=
    \int_0^1
    \left(
        -\frac{a(t)}{\rho(t)^2}
    \right)
    \eta(t)
    \,dt
    +
    \lambda
    \int_0^1
    \eta(t)
    \,dt \\
    &=
    \int_0^1
    \left(
        -\frac{a(t)}{\rho(t)^2}
        +
        \lambda
    \right)
    \eta(t)
    \,dt.
    \label{eq:app_first_variation}
\end{align}
At an interior constrained optimum, the first variation of the Lagrangian must
vanish for all smooth perturbations \(\eta\). Therefore,
\begin{equation}
    -\frac{a(t)}{\rho(t)^2}
    +
    \lambda
    =
    0
    \qquad
    \text{for all } t\in[0,1].
    \label{eq:app_euler_lagrange}
\end{equation}
Rearranging gives
\begin{equation}
    \rho(t)^2
    =
    \frac{a(t)}{\lambda}.
\end{equation}
Since \(\rho(t)>0\) and \(a(t)>0\), we need \(\lambda>0\), and thus
\begin{equation}
    \rho(t)
    =
    \frac{1}{\sqrt{\lambda}}
    \sqrt{a(t)}.
    \label{eq:app_rho_unscaled}
\end{equation}
The normalization constraint determines \(\lambda\). Substituting
Eq.~\eqref{eq:app_rho_unscaled} into
\(\int_0^1 \rho(t)\,dt=1\) gives
\begin{align}
    1
    &=
    \int_0^1
    \frac{1}{\sqrt{\lambda}}
    \sqrt{a(t)}
    \,dt \\
    &=
    \frac{1}{\sqrt{\lambda}}
    \int_0^1
    \sqrt{a(t)}
    \,dt.
\end{align}
Therefore,
\begin{equation}
    \sqrt{\lambda}
    =
    \int_0^1
    \sqrt{a(t)}
    \,dt,
\end{equation}
and hence
\begin{equation}
    \rho^\star(t)
    =
    \frac{\sqrt{a(t)}}
    {\int_0^1 \sqrt{a(s)}\,ds}.
\end{equation}
This proves the claimed form of the optimizer.

We next compute the minimum value. Let
$Z
    =
    \int_0^1
    \sqrt{a(s)}
    \,ds.$
Then
\begin{equation}
    \rho^\star(t)
    =
    \frac{\sqrt{a(t)}}{Z}.
\end{equation}
Substituting into the objective gives
\begin{align}
    J[\rho^\star]
    &=
    \int_0^1
    \frac{a(t)}{\rho^\star(t)}
    \,dt =
    \int_0^1
    \frac{a(t)}
    {
        \sqrt{a(t)}/Z
    }
    \,dt =
    Z
    \int_0^1
    \sqrt{a(t)}
    \,dt =
    Z^2 =
    \left(
        \int_0^1
        \sqrt{a(t)}
        \,dt
    \right)^2.
\end{align}

It remains to justify uniqueness. For each fixed \(t\), the function
$r \mapsto \frac{a(t)}{r}$
is strictly convex on \(r>0\) since
\begin{equation}
    \frac{d^2}{dr^2}
    \left(
        \frac{a(t)}{r}
    \right)
    =
    \frac{2a(t)}{r^3}
    >
    0.
\end{equation}
Since \(a(t)>0\), the integral functional is strictly convex over positive
densities. The feasible set
\begin{equation}
    \left\{
        \rho:\rho(t)>0,\ \int_0^1\rho(t)\,dt=1
    \right\}
\end{equation}
is convex. Therefore, any stationary point satisfying the constraint is the
unique global minimizer. This completes the proof.
\end{proof}

We now state a discrete version of the minimization problem  and characterize the corresponding solution.

\begin{lemma}
\label{lem:sqrt_weights_minimize_empirical_proxy}
Define the empirical effective proxy
\begin{equation}
    \widehat{\mathcal J}_{\sigma,\varepsilon_a}(w)
    =
    \sum_{i=0}^{N-1}
    \frac{
        \widehat\phi_i\Delta_i^2
    }{
        w_i
    }.
\end{equation}
The unique minimizer over the simplex $w_i>0$, $\sum_i w_i=1$
is
\begin{equation}
    \widehat w_i
    =
    \frac{
        \Delta_i\widehat b_i
    }{
        \sum_{\ell=0}^{N-1}
        \Delta_\ell\widehat b_\ell
    }.
\end{equation}
Under Assumption (A3), this minimizer lies in $\mathcal W_\omega$, and therefore
it is also the minimizer over $\mathcal W_\omega$.
\end{lemma}

\begin{proof}
The proof is the discrete analogue of Proposition~\ref{prop:sharpness_density},
applied to the coefficients \(\widehat\phi_i=\widehat b_i^2\).

Since \(\widehat\phi_i=\widehat b_i^2\), the Lagrangian is
\begin{equation}
    \mathcal L(w,\lambda)
    =
    \sum_i
    \frac{
        \widehat b_i^2\Delta_i^2
    }{
        w_i
    }
    +
    \lambda
    \left(
        \sum_iw_i-1
    \right),
\end{equation}
where $\lambda$ is the Lagrange multiplier.
The first-order optimality condition is
\begin{equation}
    -
    \frac{\widehat b_i^2\Delta_i^2}{w_i^2}
    +
    \lambda
    =
    0.
\end{equation}
Hence, $w_i^2 =    \frac{\widehat b_i^2\Delta_i^2}{\lambda}.$
Since $w_i>0$,
\begin{equation}
    w_i
    =
    \frac{\Delta_i\widehat b_i}{\sqrt\lambda}.
\end{equation}
Using $\sum_iw_i=1$,
\begin{equation}
    \sqrt\lambda
    =
    \sum_{\ell=0}^{N-1}
    \Delta_\ell\widehat b_\ell.
\end{equation}
Therefore,
\begin{equation}
    w_i
    =
    \frac{
        \Delta_i\widehat b_i
    }{
        \sum_{\ell=0}^{N-1}
        \Delta_\ell\widehat b_\ell
    }.
\end{equation}
The objective is strictly convex on the positive simplex because
$w\mapsto 1/w$ is strictly convex on $w>0$. Therefore the minimizer is unique.
Since $\mathcal W_\omega$ is a subset of the open simplex and the unconstrained
global minimizer belongs to $\mathcal W_\omega$ under Assumption (A3), it is also
the constrained minimizer over $\mathcal W_\omega$.
\end{proof}

We next recall Bernstein's inequality (see, e.g., Proposition 1.4 in \cite{bach2024learning}) in the form used below.

\begin{lemma}[Bernstein's inequality]
\label{lem:bernstein_inequality}
Let $Z_1,\ldots,Z_n$ be independent random variables satisfying
\begin{equation}
    \mathbb E[Z_j]=0,
    \qquad
    |Z_j|\leq c
    \quad
    \text{almost surely}.
\end{equation}
Let $s^2 = \frac1n \sum_{j=1}^n    \operatorname{Var}(Z_j).$
Then, for every $t > 0$,
\begin{equation}
    \mathbb P
    \left(
        \left|
            \frac1n
            \sum_{j=1}^n
            Z_j
        \right|
        \geq 
        t
    \right)
    \leq
    2\exp
    \left(
        -
        \frac{
            nt^2
        }{
            2s^2+\frac{2ct}{3}
        }
    \right).
    \label{eq:bernstein_tail}
\end{equation}
Moreover, for every $\delta\in(0,1)$, with probability at least
$1-\delta$,
\begin{equation}
    \left|
        \frac1n
        \sum_{j=1}^n
        Z_j
    \right|
    \leq
    \sqrt{
        \frac{2s^2\log(2/\delta)}{n}
    }
    +
    \frac{2c\log(2/\delta)}{3n}.
    \label{eq:bernstein_high_probability}
\end{equation}
The same statement applies to noncentered variables after replacing
$Z_j$ by $Z_j-\mathbb E[Z_j]$.
\end{lemma}

Using this inequality and our assumptions, we obtain the following auxiliary result. 

\begin{lemma}
\label{lem:bernstein_effective_profile}
With probability at least $1-\delta$,
\begin{equation}
    \max_i
    \left|
        \widehat\phi_i-\phi_i^\star
    \right|
    \leq
    \zeta_M(\delta).
\end{equation}
\end{lemma}

\begin{proof}
We first prove concentration of the raw empirical profile
\(\widehat{\mathcal A}_i\). Fix \(i\in\{0,\ldots,N-1\}\) and define
\begin{equation}
    U_i^{(m)}
    :=
    A_i^{(m)}-\mathcal A_i,
    \qquad
    m=1,\ldots,M.
\end{equation}
By Assumption (A1), the variables
\(U_i^{(1)},\ldots,U_i^{(M)}\) are independent. By Assumption (A2),
\begin{equation}
    \mathbb E[U_i^{(m)}]=0,
    \qquad
    |U_i^{(m)}|\leq R,
    \qquad
    \operatorname{Var}(U_i^{(m)})\leq \nu^2.
\end{equation}
for some positive constants $R, \nu$. Moreover,
\begin{equation}
    \widehat{\mathcal A}_i-\mathcal A_i
    =
    \frac1M
    \sum_{m=1}^M
    U_i^{(m)}.
\end{equation}
Therefore, by Bernstein's inequality (Lemma \ref{lem:bernstein_inequality}), for every \(\eta>0\),
\begin{equation}
    \mathbb P
    \left(
        \left|
            \widehat{\mathcal A}_i-\mathcal A_i
        \right|
        \geq \eta
    \right)
    \leq
    2\exp
    \left(
        -
        \frac{
            M\eta^2
        }{
            2\nu^2+\frac23R\eta
        }
    \right).
\end{equation}
Taking a union bound over \(i=0,\ldots,N-1\) gives
\begin{equation}
    \mathbb P
    \left(
        \max_i
        \left|
            \widehat{\mathcal A}_i-\mathcal A_i
        \right|
        \geq \eta
    \right)
    \leq
    2N\exp
    \left(
        -
        \frac{
            M\eta^2
        }{
            2\nu^2+\frac23R\eta
        }
    \right).
\end{equation}
Set $u=\log(2N/\delta)$.
The right-hand side is at most \(\delta\) for
\begin{equation}
    \eta
    =
    \sqrt{
        \frac{2\nu^2u}{M}
    }
    +
    \frac{2Ru}{3M}.
\end{equation}
Substituting \(u=\log(2N/\delta)\) gives
\begin{equation}
    \max_i
    \left|
        \widehat{\mathcal A}_i-\mathcal A_i
    \right|
    \leq
    \eta_M(\delta)
\end{equation}
with probability at least \(1-\delta\).

On this event, we next pass from the raw profile to the square-root profile.
Since the map
$r\mapsto \sqrt{r+\varepsilon_a}$
is \(1/(2\sqrt{\varepsilon_a})\)-Lipschitz on \(r\geq0\), we have
\begin{align}
    \max_i
    |\widehat g_i-g_i^\star|
    &=
    \max_i
    \left|
        \sqrt{\widehat{\mathcal A}_i+\varepsilon_a}
        -
        \sqrt{\mathcal A_i+\varepsilon_a}
    \right| \\
    &\leq
    \frac{1}{2\sqrt{\varepsilon_a}}
    \max_i
    |\widehat{\mathcal A}_i-\mathcal A_i| \leq
    \frac{\eta_M(\delta)}{2\sqrt{\varepsilon_a}}
    =
    \alpha_M(\delta).
\end{align}

Now we pass through the smoothing operator. Since
\(K_\sigma(i,\ell)\geq0\) and \(\sum_{\ell=0}^{N-1}K_\sigma(i,\ell)=1\), the
smoothing operator is nonexpansive in the sup norm, so:
\begin{align}
    |\widehat b_i-b_i^\star|
    &=
    \left|
        \sum_{\ell=0}^{N-1}
        K_\sigma(i,\ell)
        \left(
            \widehat g_\ell-g_\ell^\star
        \right)
    \right| \leq
    \sum_{\ell=0}^{N-1}
    K_\sigma(i,\ell)
    |\widehat g_\ell-g_\ell^\star| \leq
    \max_\ell
    |\widehat g_\ell-g_\ell^\star|
    \leq
    \alpha_M(\delta).
\end{align}
Therefore, $\max_i
    |\widehat b_i-b_i^\star|
    \leq
    \alpha_M(\delta).$

Finally, recall that
$\widehat\phi_i=\widehat b_i^2,$ $\phi_i^\star=(b_i^\star)^2.$
Using Assumption (A3), i.e.
\(0\leq \widehat b_i\leq G_{\max}\) and
\(0\leq b_i^\star\leq G_{\max}\), we get
\begin{align}
    |\widehat\phi_i-\phi_i^\star|
    &=
    |\widehat b_i^2-(b_i^\star)^2| =
    |\widehat b_i-b_i^\star|
    |\widehat b_i+b_i^\star| \leq
    2G_{\max}
    |\widehat b_i-b_i^\star| \leq
    2G_{\max}\alpha_M(\delta)
    =
    \zeta_M(\delta).
\end{align}
Taking the maximum over \(i\) gives
$\max_i
    |\widehat\phi_i-\phi_i^\star|
    \leq
    \zeta_M(\delta)$
with probability at least $1-\delta$. This completes the proof.
\end{proof}

Next, we will need the following series of lemmas. 
\begin{lemma}
\label{lem:smoothed_oracle_inequality}
Let
\begin{equation}
    w_{\sigma,\varepsilon_a}^\star
    =
    \arg\min_{w\in\mathcal W_\omega}
    \mathcal J_{\sigma,\varepsilon_a}(w).
\end{equation}
On the event
\begin{equation}
    \max_i
    |\widehat\phi_i-\phi_i^\star|
    \leq
    \zeta,
\end{equation}
we have
\begin{equation}
    \mathcal J_{\sigma,\varepsilon_a}(\widehat w)
    \leq
    \mathcal J_{\sigma,\varepsilon_a}(w_{\sigma,\varepsilon_a}^\star)
    +
    \frac{2D_2}{\omega}\zeta.
\end{equation}
\end{lemma}

\begin{proof}
For any $w\in\mathcal W_\omega$,
\begin{align}
    \left|
        \widehat{\mathcal J}_{\sigma,\varepsilon_a}(w)
        -
        \mathcal J_{\sigma,\varepsilon_a}(w)
    \right|
    =
    \left|
        \sum_i
        \frac{
            (\widehat\phi_i-\phi_i^\star)\Delta_i^2
        }{
            w_i
        }
    \right| 
    \leq
    \sum_i
    \frac{
        |\widehat\phi_i-\phi_i^\star|\Delta_i^2
    }{
        w_i
    } 
    \leq
    \frac{\zeta}{\omega}
    \sum_i\Delta_i^2 
    =
    \frac{D_2}{\omega}\zeta.
\end{align}
Therefore,
\begin{equation}
    \sup_{w\in\mathcal W_\omega}
    \left|
        \widehat{\mathcal J}_{\sigma,\varepsilon_a}(w)
        -
        \mathcal J_{\sigma,\varepsilon_a}(w)
    \right|
    \leq
    \frac{D_2}{\omega}\zeta.
\end{equation}
By Lemma~\ref{lem:sqrt_weights_minimize_empirical_proxy} and Assumption (A3),
$\widehat w$ minimizes $\widehat{\mathcal J}_{\sigma,\varepsilon_a}$ over
$\mathcal W_\omega$. Hence
\begin{equation}
    \widehat{\mathcal J}_{\sigma,\varepsilon_a}(\widehat w)
    \leq
    \widehat{\mathcal J}_{\sigma,\varepsilon_a}
    (w_{\sigma,\varepsilon_a}^\star).
\end{equation}
Thus,
\begin{align}
    \mathcal J_{\sigma,\varepsilon_a}(\widehat w)
    \leq
    \widehat{\mathcal J}_{\sigma,\varepsilon_a}(\widehat w)
    +
    \frac{D_2}{\omega}\zeta 
    \leq
    \widehat{\mathcal J}_{\sigma,\varepsilon_a}
    (w_{\sigma,\varepsilon_a}^\star)
    +
    \frac{D_2}{\omega}\zeta 
    \leq
    \mathcal J_{\sigma,\varepsilon_a}
    (w_{\sigma,\varepsilon_a}^\star)
    +
    \frac{2D_2}{\omega}\zeta.
\end{align}
\end{proof}

\begin{lemma}
\label{lem:smoothed_to_raw_risk}
For any $w\in\mathcal W_\omega$,
\begin{equation}
    |\mathcal J_{\sigma,\varepsilon_a}(w)-\mathcal J(w)|
    \leq
    \frac{D_2}{\omega}
    \beta_{\sigma,\varepsilon_a}.
\end{equation}
Consequently, on the event of Lemma~\ref{lem:smoothed_oracle_inequality},
\begin{equation}
    \mathcal J(\widehat w)
    \leq
    \mathcal J(w^\star)
    +
    \frac{2D_2}{\omega}
    (\zeta+\beta_{\sigma,\varepsilon_a}).
\end{equation}
\end{lemma}

\begin{proof}
For any $w\in\mathcal W_\omega$,
\begin{align}
    |\mathcal J_{\sigma,\varepsilon_a}(w)-\mathcal J(w)|
    =
    \left|
        \sum_i
        \frac{
            (\phi_i^\star-\mathcal A_i)\Delta_i^2
        }{
            w_i
        }
    \right| 
    \leq
    \sum_i
    \frac{
        |\phi_i^\star-\mathcal A_i|\Delta_i^2
    }{
        w_i
    } 
    \leq
    \frac{\beta_{\sigma,\varepsilon_a}}{\omega}
    \sum_i\Delta_i^2 
    =
    \frac{D_2}{\omega}
    \beta_{\sigma,\varepsilon_a}.
\end{align}
Now,
\begin{equation}
    \mathcal J(\widehat w)
    \leq
    \mathcal J_{\sigma,\varepsilon_a}(\widehat w)
    +
    \frac{D_2}{\omega}\beta_{\sigma,\varepsilon_a}.
\end{equation}
By Lemma~\ref{lem:smoothed_oracle_inequality},
\begin{equation}
    \mathcal J_{\sigma,\varepsilon_a}(\widehat w)
    \leq
    \mathcal J_{\sigma,\varepsilon_a}(w_{\sigma,\varepsilon_a}^\star)
    +
    \frac{2D_2}{\omega}\zeta.
\end{equation}
Since $w_{\sigma,\varepsilon_a}^\star$ minimizes
$\mathcal J_{\sigma,\varepsilon_a}$ over $\mathcal W_\omega$,
\begin{equation}
    \mathcal J_{\sigma,\varepsilon_a}(w_{\sigma,\varepsilon_a}^\star)
    \leq
    \mathcal J_{\sigma,\varepsilon_a}(w^\star).
\end{equation}
Using the bias bound again,
\begin{equation}
    \mathcal J_{\sigma,\varepsilon_a}(w^\star)
    \leq
    \mathcal J(w^\star)
    +
    \frac{D_2}{\omega}\beta_{\sigma,\varepsilon_a}.
\end{equation}
Combining the above gives
\begin{equation}
    \mathcal J(\widehat w)
    \leq
    \mathcal J(w^\star)
    +
    \frac{2D_2}{\omega}\zeta
    +
    \frac{2D_2}{\omega}\beta_{\sigma,\varepsilon_a}.
\end{equation}
Therefore,
\begin{equation}
    \mathcal J(\widehat w)
    \leq
    \mathcal J(w^\star)
    +
    \frac{2D_2}{\omega}
    (\zeta+\beta_{\sigma,\varepsilon_a}).
\end{equation}
\end{proof}

\begin{lemma}
\label{lem:risk_controls_distribution}
For any $w\in\mathcal W_\omega$,
\begin{equation}
    W_1(\mu_w,\mu_1)
    \leq
    \frac{C_{\rm Eul}}{B}
    \mathcal J(w).
\end{equation}
\end{lemma}

\begin{proof}
By the  definition of the 1-Wasserstein distance,
\begin{equation}
    W_1(\mu_w,\mu_1)
    =
    \inf_{\pi\in\Pi(\mu_w,\mu_1)}
    \int \|x-y\|\,d\pi(x,y), 
\end{equation}
where $\Pi(\mu,\nu)$ denotes the set of all couplings of $\mu$ and $\nu$. 
Using the shared-noise coupling
\begin{equation}
    \pi_w=(T_w,T)_{\#}p_0,
\end{equation}
we obtain
\begin{equation}
    W_1(\mu_w,\mu_1)
    \leq
    \mathbb E_{x_0\sim p_0}
    \|T_w(x_0)-T(x_0)\|.
\end{equation}

By Assumption (A4),
\begin{equation}
    \|T_w(x_0)-T(x_0)\|
    \leq
    \frac{C_{\rm Eul}}{B}
    \sum_i
    \frac{
        A_i(x_0)\Delta_i^2
    }{
        w_i
    }.
\end{equation}
Taking expectation over $x_0\sim p_0$,
\begin{equation}
    \mathbb E
    \|T_w(x_0)-T(x_0)\|
    \leq
    \frac{C_{\rm Eul}}{B}
    \mathbb E
    \left[
        \sum_i
        \frac{
            A_i(x_0)\Delta_i^2
        }{
            w_i
        }
    \right].
\end{equation}
The expectation on the right is exactly $\mathcal J(w)$. Therefore,
\begin{equation}
    W_1(\mu_w,\mu_1)
    \leq
    \frac{C_{\rm Eul}}{B}
    \mathcal J(w).
\end{equation}
\end{proof}

We now have all the needed ingredients to prove Theorem~\ref{thm:stat_guarantee}. 

\begin{proof}[Proof of Theorem~\ref{thm:stat_guarantee}]
By Lemma~\ref{lem:bernstein_effective_profile}, with probability at least
$1-\delta$,
\begin{equation}
    \max_i
    |\widehat\phi_i-\phi_i^\star|
    \leq
    \zeta_M(\delta).
\end{equation}
On this high-probability event, Lemma~\ref{lem:smoothed_to_raw_risk} gives
\begin{equation}
    \mathcal J(\widehat w)
    \leq
    \mathcal J(w^\star)
    +
    \frac{2D_2}{\omega}
    \left(
        \zeta_M(\delta)+\beta_{\sigma,\varepsilon_a}
    \right).
\end{equation}
By Lemma~\ref{lem:risk_controls_distribution},
\begin{equation}
    W_1(\mu_{\widehat w},\mu_1)
    \leq
    \frac{C_{\rm Eul}}{B}
    \mathcal J(\widehat w).
\end{equation}
Combining the two inequalities,
\begin{equation}
    W_1(\mu_{\widehat w},\mu_1)
    \leq
    \frac{C_{\rm Eul}}{B}
    \left[
        \mathcal J(w^\star)
        +
        \frac{2D_2}{\omega}
        \left(
            \zeta_M(\delta)+\beta_{\sigma,\varepsilon_a}
        \right)
    \right].
\end{equation}
This completes the proof of the main theorem.
\end{proof}

\section{Experimental Details}
\label{app:exp_details}

\subsection{Synthetic Datasets and Model Training}
\label{app:synthetic_details}

We evaluate SharpEuler on two synthetic two-dimensional distributions:
\emph{Branched Manifold} and \emph{Rotated Grid}. These datasets are chosen to
distinguish two failure modes of low-budget Euler sampling. Branched Manifold tests
whether the sampler resolves thin, curved support, while Rotated Grid tests
whether it preserves separated modes without placing excessive mass between
them.

For each dataset, we train a small residual MLP flow matching model whose probability path is described by the linear
interpolation
\[
    x_t = (1-t)x_0 + t x_1, \qquad t \in [0,1],
\]
where \(x_0\) is a noise sample and \(x_1\) is a data sample. The regression
target is
\[
    v^* = x_1 - x_0,
\]
corresponding to the forward-time convention described in the main paper, integrating from  noise to data. At inference, we
sample with the corresponding reverse-time Euler update. All reported samples
use EMA weights. The same trained model, random seed, and initial noise samples
are used across schedules whenever possible, so that differences are attributable
to the timestep grid.

\subsection{Schedule Baselines}
\label{app:schedule_baselines}

We compare schedules at fixed number of function evaluations (NFE), with budgets
\(B\in\{8,12,16,20\}\) for the synthetic experiments. The sampler and model are
unchanged across schedules.

\paragraph{Uniform.}
The Uniform baseline uses an equally spaced descending reverse-time grid from
\(s=1\) to \(s=0\).

\paragraph{ETS.}
The ETS baseline is the entropy-driven schedule of
\cite{stancevic2025entropic}, adapted to our flow matching predictor following
the BioNeMo MoCo implementation~\citep{bionemo_ets_tutorial}. We use the
absolute divergence when constructing the cumulative entropy estimate to ensure
a monotone schedule.

\paragraph{SharpEuler.}
SharpEuler constructs a profile-aware density
\[
    \rho(t) \propto \phi(t)^\gamma,
\]
where \(\phi(t)\) is the smoothed finite-difference acceleration profile
estimated from calibration trajectories. We evaluate
\(\gamma\in\{0.5,1.0,1.5\}\). The choice \(\gamma=0.5\) corresponds to the
LTE-equidistributing density derived in Sec.~\ref{sec:method}. Larger values
concentrate the grid more strongly in regions assigned high sharpness by the
calibration profile.

\subsection{Evaluation Metrics}
\label{app:metrics}

For the synthetic experiments, we report Density and Coverage
\citep{naeem2020reliable}, which are robust alternatives to the
precision-recall metrics of \cite{kynkaanniemi2019improved}. Both metrics are
computed using \(k\)-nearest-neighbor (NN) balls around real samples with \(k=5\).

\paragraph{Density.}
For each generated sample, Density counts how many real samples \(k\)-NN balls contain, averaged over generated samples and normalized by \(k\). Higher Density indicates that generated samples lie in regions of high real sample density.

\paragraph{Coverage.}
Coverage is the fraction of real samples whose \(k\)-NN ball contains at least one generated sample. Higher Coverage indicates that the generator reaches a larger fraction  of the target support.

\paragraph{Wasserstein-2 distance.}
We also report squared Wasserstein-2 distance $W_2^2$, estimated with
Sinkhorn-regularized optimal transport~\citep{villani2008optimal,cuturi2013sinkhorn}.
We use \(n=2000\) samples from each distribution and entropic regularization \(\varepsilon=0.01\). Lower \(W_2^2\) indicates a closer distributional match.

\subsection{FLUX Calibration and Inference Details}
\label{app:flux_details}

We evaluate FLUX.1-dev~\citep{flux2024}, a 12B-parameter rectified flow
transformer for text-to-image generation, at its native \(1024{\times}1024\) resolution. All samples are generated in \texttt{bfloat16} precision with classifier-free guidance scale \(3.5\).

We calibrate SharpEuler once using 64 prompts spanning landscapes, urban scenes, wildlife, and abstract content. For each prompt, we run the default 50-step FLUX.1-dev pipeline in \texttt{diffusers}~\citep{diffusers_flux_pipeline}. During
calibration, we record the velocity predictions used by \texttt{scheduler.step},
take finite differences of consecutive predictions to estimate acceleration along the trajectory, average these estimates across prompts, and smooth the resulting profile with a Gaussian kernel of bandwidth \(\sigma=1.0\) on the 49-point midpoint grid. The calibrated profile is then fixed and reused for all held-out prompts and all inference budgets.

We compare SharpEuler with \(\gamma=0.5\) against two profile-agnostic
schedules. The first is the default FLUX pipeline schedule, which applies the resolution-dependent \(\mu\)-shift used by \texttt{diffusers}. The second is a static shifted schedule with \(\alpha=3\), following the SD3-style FlowMatch Euler shift~\citep{esser2024scaling},
\[
    t_{\mathrm{shifted}}
    =
    \frac{\alpha t}{1+(\alpha-1)t}.
\]
For SharpEuler, we use a tail-aligned variant in which the final step is
replaced by the corresponding pipeline step. This preserves the hand-tuned low-noise tail of the FLUX sampler, where the calibrated finite-difference profile is empirically noisier.

We omit ETS from the FLUX experiments because its flow matching adaptation requires Jacobian-vector products through the transformer at candidate timesteps, which is too expensive at \(1024{\times}1024\) resolution.

\subsection{GPT-5.5 Vision-Judge Protocol}
\label{app:judge_protocol}

We use GPT-5.5 as a vision judge to compare SharpEuler and the default FLUX pipeline schedule against the same 50-step reference image. For each prompt, the judge receives four inputs: the text prompt, the 50-step reference image, the SharpEuler image, and the Pipeline image. The two candidate images are shown in randomized order to reduce position bias.

The judge returns two types of output. First, it assigns each candidate an independent integer score from 1 to 10, where higher scores indicate closer agreement with the reference image in the intended generative sense. Second, it returns a forced pairwise preference, allowing ties only when the candidates are effectively indistinguishable. The rubric instructs the judge to compare prompt-grounded semantic content, global composition, scene layout, object identity, style, viewpoint, lighting, color palette, and salient spatial relationships. It also instructs the judge not to over-penalize harmless stochastic differences or exact pixel-level deviations.

For each budget, we report the mean SharpEuler score, mean Pipeline score, the paired score difference
($\Delta
    =
    \text{score}_{\mathrm{SharpEuler}}
    -
    \text{score}_{\mathrm{Pipeline}})$, 
the standard deviation of \(\Delta\) across prompts, a paired \(t\)-test
\(p\)-value for the score differences, and the number of pairwise wins for SharpEuler. Statistical tests are computed over prompt-level paired
comparisons.

\section{Additional Results}
\label{app:results}

\subsection{Sharpness Profiles}

Figure~\ref{fig:sharpness_profiles} shows calibrated sharpness profiles for three synthetic datasets: Branched Manifold (left), Rotated Grid (middle), and Spiral (right). The top row shows the target data distributions. The middle row shows the calibrated sharpness profile together with the resulting timestep grid for \(\gamma=0.5\) at budget \(B=16\), while the bottom row shows the
corresponding profile and grid for \(\gamma=1.0\).

The profiles differ across datasets, indicating that the
calibrated sharpness signal depends on both the learned model and the geometry of the target distribution. Thus, the induced timestep grids are also dataset dependent. Increasing \(\gamma\) concentrates grid points more strongly in regions assigned with high sharpness, producing visibly different timestep placements even at the same evaluation budget.
Since the calibrated profiles are derived from the local behavior of the learned flow, an interesting direction for future work is to study whether the induced timestep densities correlate with uncertainty or ambiguity in the generative
trajectory~\cite{gupta2026quantifying}.

\begin{figure}[!t]
    \centering
    \includegraphics[width=0.9\textwidth]{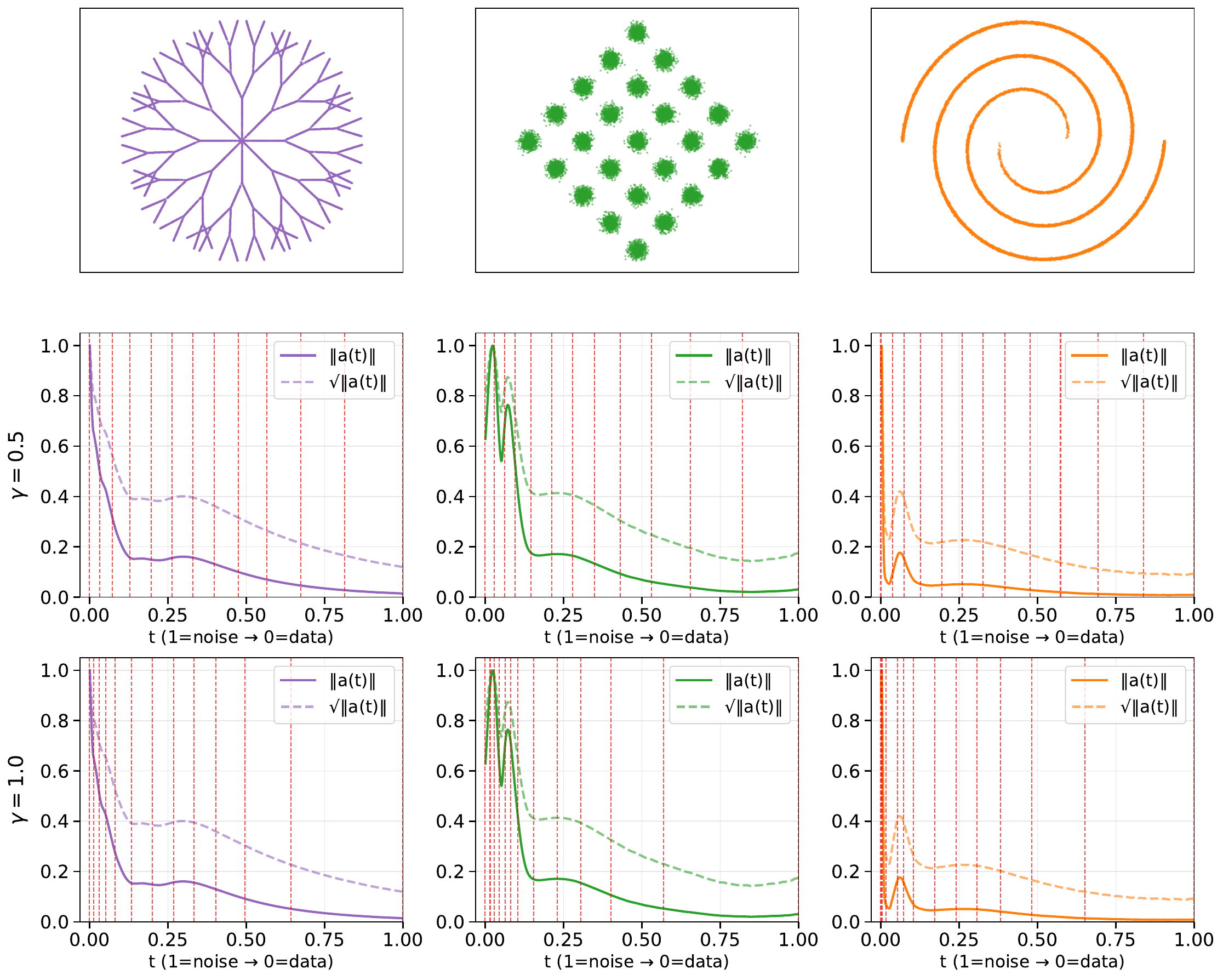}
    \caption{Calibrated sharpness profiles and induced timestep grids for Branched
Manifold (left), Rotated Grid (middle), and Spiral (right). The top row shows
target samples. The middle and bottom rows show the sharpness profile and
resulting \(B=16\) grid for \(\gamma=0.5\) and \(\gamma=1.0\), respectively.
The profiles and induced grids vary across datasets, and larger \(\gamma\)
concentrates timestep placements more strongly in high-sharpness regions.}
    \label{fig:sharpness_profiles}
\end{figure}

\subsection{Gamma Ablation}

We ablate the exponent \(\gamma\) in the schedule density
\(\rho(t)\propto\phi(t)^\gamma\) on Branched Manifold and Rotated Grid. We sweep \(\gamma\in\{0.5,0.8,1.0,1.2,1.5,1.7,2.0\}\) and evaluate Density and Coverage at budgets \(B\in\{5,8,12,16,20\}\). Results are shown in \Cref{fig:gamma_sweep}.

The sweep shows a consistent fidelity--coverage tradeoff. Increasing
\(\gamma\) tends to increase Density, because the resulting grid concentrates more evaluations at times with high estimated sharpness. Coverage behaves differently: it usually peaks at an intermediate value, often around \(\gamma\in[1.0,1.2]\), and decreases for larger \(\gamma\). This indicates that overly concentrated grids can improve local sample fidelity while leaving parts of the target support under-sampled. The effect is strongest at the smallest budget, \(B=5\), where each timestep placement has high leverage, and weakens as the budget increases. Overall, \(\gamma=0.5\) remains the LTE-prescribed default, while \(\gamma\in[1.0,1.5]\) often gives the best empirical Density--Coverage tradeoff on these synthetic datasets.

\begin{figure}[!h]
    \centering
    \begin{subfigure}[t]{\textwidth}
        \centering
        \includegraphics[width=\textwidth]{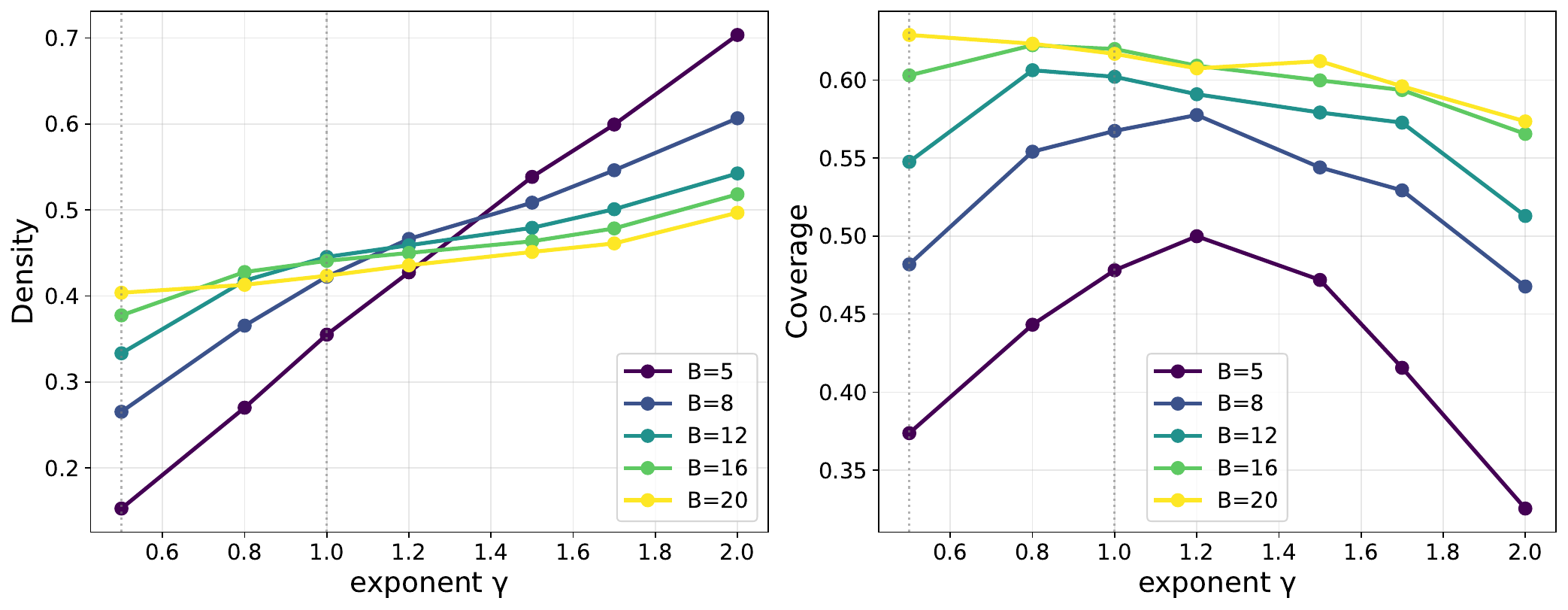}
        \caption{Branched Manifold}
        \label{fig:gamma_sweep_branchy}
    \end{subfigure}
    \\[0.6em]
    \begin{subfigure}[t]{\textwidth}
        \centering
        \includegraphics[width=\textwidth]{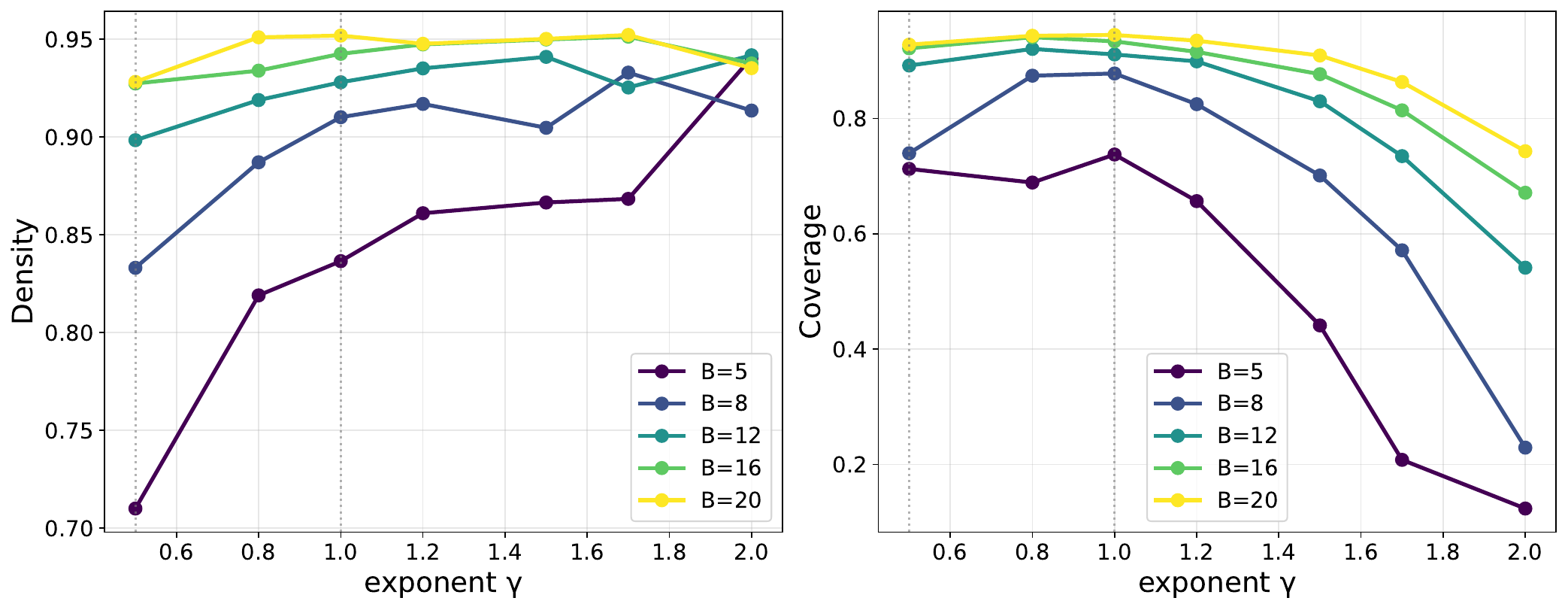}
        \caption{Rotated Grid}
        \label{fig:gamma_sweep_rotgrid}
    \end{subfigure}
\caption{Density and Coverage as functions of the exponent \(\gamma\) for Branched Manifold and Rotated Grid, across budgets
\(B\in\{5,8,12,16,20\}\). Increasing \(\gamma\) generally improves Density, while Coverage is maximized at intermediate values and can decrease when the grid becomes too concentrated. Dashed vertical lines mark \(\gamma=0.5\), the LTE-prescribed choice, and \(\gamma=1.0\).}
\label{fig:gamma_sweep}
\end{figure}

\newpage 
\section*{Broader Impact}
This work develops a sampling method for flow-based generative models. The contribution is primarily methodological: SharpEuler changes the placement of Euler evaluations at inference time and does not introduce new training data, model capabilities, or deployment mechanisms. Nevertheless, improvements in low-budget sampling can reduce the cost of generating images and other synthetic data, which may amplify both beneficial and harmful uses of generative models.

The method should therefore be used with the same safeguards as the underlying generative model. In applications where generated outputs may affect scientific, medical, legal, or safety-critical decisions, low-budget samples should not be treated as verified outputs. Reducing the number of sampling steps can also change failure modes, introduce artifacts, or alter coverage of the target distribution. Human review and task-specific validation remain necessary when sample correctness or semantic fidelity matters.

\end{document}